%% file: example_paper.tex
\documentclass{article}
\input{math_commands.tex}

\usepackage{microtype}
\usepackage{graphicx}
\usepackage{subfigure}
\usepackage{booktabs} 
\usepackage{hyperref}

\usepackage{algorithm}
\usepackage{algorithmic}
\usepackage{times}

\usepackage[accepted]{icml2025}
\usepackage{amsmath}
\usepackage{amssymb}
\usepackage{mathtools}
\usepackage{amsthm}

\usepackage[capitalize,noabbrev]{cleveref}

\usepackage[textsize=tiny]{todonotes}

\icmltitlerunning{On the Expressivity of Selective State-Space Layers}

\begin{document}

\twocolumn[
\icmltitle{On the Expressivity of Selective State-Space Layers:\\ A Multivariate Polynomial Approach}



\icmlsetsymbol{equal}{*}

\begin{icmlauthorlist}
\icmlauthor{Edo Cohen-Karlik}{equal,yyy}
\icmlauthor{Itamar Zimerman}{equal,yyy}
\icmlauthor{Liane Galanti}{equal,yyy}
\icmlauthor{Ido Atad}{yyy}
\icmlauthor{Amir Globerson}{yyy}
\icmlauthor{Lior Wolf}{yyy}
\end{icmlauthorlist}

\icmlaffiliation{yyy}{The School of Computer Science, Tel Aviv University}

\icmlcorrespondingauthor{Edo Cohen-Karlik}{edocohen@mail.tau.ac.il}

\icmlkeywords{Mamba, S6, Expressivity, Generelization}

\vskip 0.3in
]



\printAffiliationsAndNotice{\icmlEqualContribution} 

\begin{abstract}
Recent advances in efficient sequence modeling have introduced selective state-space layers, a key component of the Mamba architecture, which have demonstrated remarkable success in a wide range of NLP and vision tasks. 
While Mamba’s empirical performance has matched or surpassed SoTA transformers on such diverse 
benchmarks, the theoretical foundations underlying its powerful representational capabilities remain less explored. In this work, we investigate the expressivity of selective state-space layers using multivariate polynomials, and prove that they surpass {\color{black}linear} transformers in expressiveness. Consequently, our findings reveal that Mamba offers superior representational power over {\color{black}linear} attention-based models for long sequences{\color{black}, while not sacrificing their generalization.} Our theoretical insights are validated by a comprehensive set of empirical experiments on various datasets.
\vspace{-6pt}
\end{abstract}

\vspace{-6pt}
\section{Introduction}
Sequence modeling has been the focus of many works in recent years, with remarkable results enabling applications such as ChatGPT. To date, these models are largely based on the Transformer architecture~\cite{NIPS2017_3f5ee243}. While transformers have proven extremely effective, they suffer from several drawbacks compared to traditional recurrent models, one of which is their computational complexity which scales quadratically with the input sequence length.

In attempt to mitigate the computational inefficiency of sequence modeling of transformers, a line of work has attempted to resurrect RNNs
, \cite{gu2021efficiently,dss,gu2021combining} 
have introduced a series of architectures called State Space Models (SSMs) which include a linear recurrence that admits efficient computations and special structure on the transition matrices.

While these architectures have demonstrated impressive performance in long sequence modeling, their performance on fundamental tasks such as language modeling fall short compared to transformers, mostly due to intrinsic superiority of transformers in modeling interactions between different elements of the input sequence. Recent work~\cite{gu2023mamba}, has introduced Mamba, an SSM variant with Selective SSMs (S6) as its core block. In an S6 layer, the parameters are a function of the input, providing the SSM with content awareness. The empirical success of Mamba is undeniable, with applications spanning large-scale language modeling~\cite{zuo2024falcon,lieber2024jamba,waleffe2024empirical}, image~\cite{mambaViT1}, 
and video~\cite{li2025videomamba} 
processing, medical imaging~\cite{mambaMedical3}, 
 tabular data analysis~\cite{ahamed2024mambatab}, Reinforcement Learning~\cite{lv2024decision}, point-cloud analysis~\cite{mambaPoint}, graph processing~\cite{mambaGraph1}
, and N-dimensional sequence modeling~\cite{mambaND}.

The success of Mamba models across various domains ignites interest in their theoretical properties. Establishing a comprehensive theoretical understanding is crucial, as it enhances our knowledge of these layers, promotes their adoption within the research community, and paves the way for future architectural advancements. Additionally, since S6 can be considered a variant of attention with linear complexity~\cite{ali2024hidden}, deeper theoretical insights could elucidate the relationships between gated RNNs, transformers, and SSMs, thereby advancing our knowledge of these interconnected architectures. 
Initial efforts to establish a theoretical framework for the expressiveness of selective (and non-selective) SSM layers have been undertaken by several researchers. Using Rough Path Theory, \citet{cirone2024theoretical} demonstrated that diagonal selective SSMs, such as Mamba, possess less expressive power than their non-diagonal counterparts. Additionally, \citet{merrill2024illusion} investigated the expressiveness relationships between SSM variants and transformers using the lens of circuit complexity, revealing that both models share the same expressive power (belonging to ${TC}^0$). Finally, \citet{jelassi2024repeat} examined the copying ability of various SSM variants compared to transformers. It concluded that from both theoretical and empirical perspectives SSMs struggle with this task. While these significant studies highlight the limited expressive capabilities of Mamba models compared to other architectures, our work introduces a different trend. We demonstrate the superior expressive power of S6 layers, using a theoretical framework based on multivariate polynomials. 

{\color{black}In addition to exploring the expressiveness gap between transformers and S6 layers, we develop norm-based length-agnostic generalization bounds for S6 layers. Suggesting that the added expressivity of the selective mechanism does not hinder the generalization properties compared to traditional SSMs as studied in \citep{liu2024generalization},  and that while polynomial expressivity increases with sequence length, it does not impact generalization.}
\noindent{\textbf{Our main contribution}} encompasses the following aspects: (i) We present simplified polynomial variants of Mamba that exhibit comparable performance on NLP and vision tasks, promoting a simpler model that can serve as a foundation for other theoretical contributions, as well as shed light on the inner dynamics of Mamba and its critical components. (ii) By establishing the connection between multivariate polynomials and S6 layers, we theoretically prove that S6 layers are expressive as a {\color{black}linear} self-attention with depth that scales logarithmically with the sequence length; that is, there are functions that can be modeled with a single S6 that would require a logarithmic number of {\color{black}linear} 
attention layers. More generally, we establish that a Mamba model with only 4 layers is sufficient to represent the class of multivariate polynomials with bounded degree L. In contrast, a {\color{black}linear} attention-based model requires a logarithmic number of layers in $L$ to achieve the same representational capability. (iii) 
Through experiments on a synthetic dataset in a controlled environment designed to isolate expressivity issues, we empirically validate that our theory is reflected in practice. (iv) {\color{black}Finally, although S6 has better polynomial expressivity for long sequences, we show that this does not come at the cost of limited generalization, by proving the first length-agnostic generalization bound for S6 
. This leads us to conclude that S6 layer possesses superior theoretical properties over linear attention for long-range tasks.}  
\begin{figure}[t]
    \centering
    \includegraphics[width=0.82\linewidth]{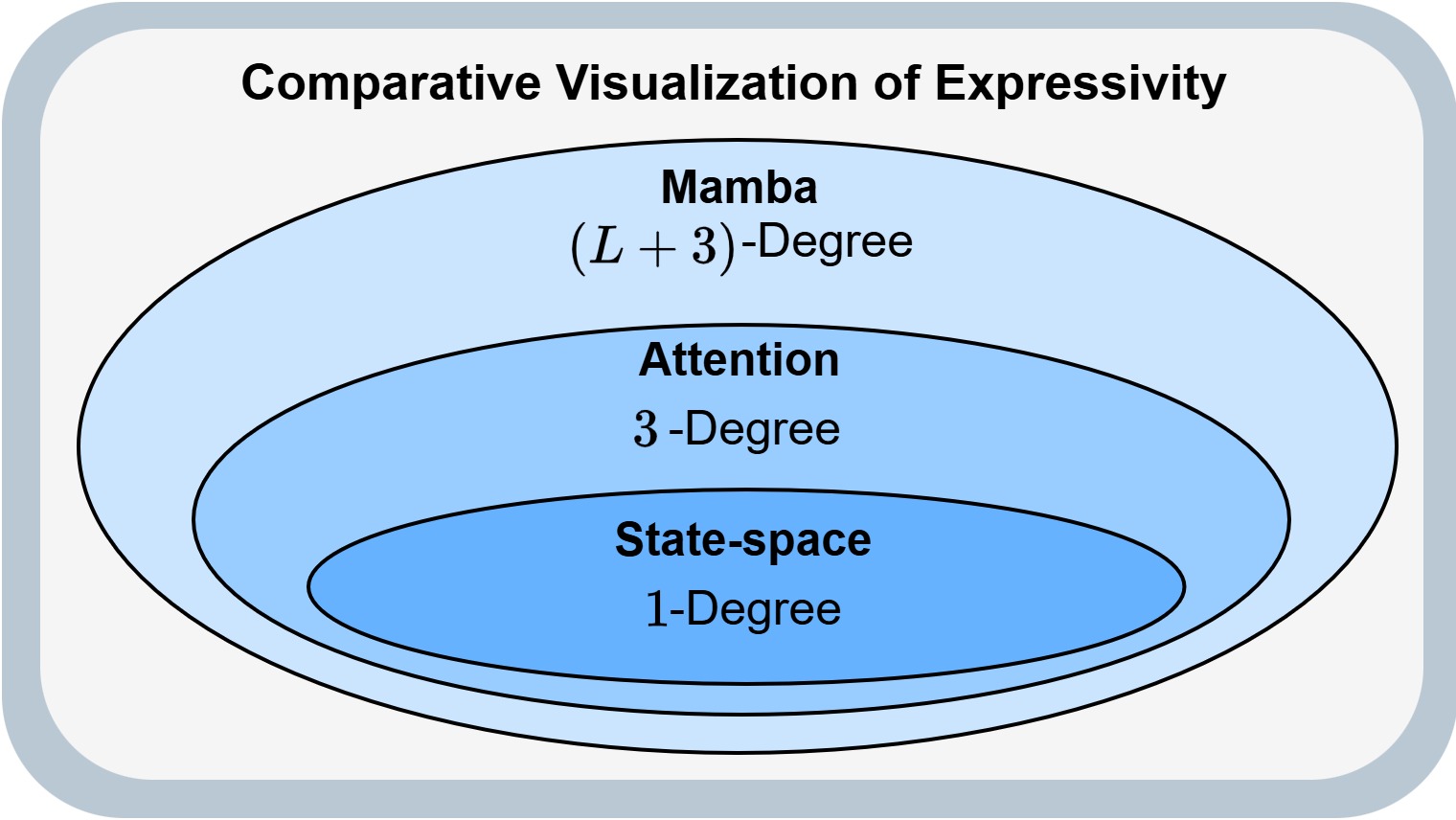}
    \vspace{-3pt}
    \caption{\textbf{Expressivity via Polynomial Degree:} Our characterization of SSMs, S6 layers, and causal self-attention via multivariate polynomials allows us to identify the expressiveness gap between these layers through maximal polynomial degree.}
    \label{fig:mainfig}
    \vspace{-12pt}
\end{figure} 
\vspace{-12pt}
\section{Related Work}
\vspace{-3pt}
SSMs have gained a lot of traction due to their remarkable performance and computational efficiency. Their theoretical properties have been the focus of many recent works; \citet{merrill2024illusion} compare the expressive capacity of (non-selective) SSMs and transformers, concluding that both architectures belong to the same complexity class ($TC^0$).\footnote{$TC^0$ is the complexity class that can be decided by polynomial-sized Boolean circuits.} \citet{sarrof2024expressive} conduct a more refined analysis and show that transformers and SSMs occupy different portions of $TC^0$. Another work showed that a single-layer Transformer with $N$ heads can simulate any state space model with $N$ channels~\citep{zimerman2023long}.

\citet{cirone2024theoretical} show that the selective mechanism introduced in Mamba results with more expressive architectures compared to traditional (non-selective) SSMs.
\citet{ali2024hidden} show that there are functions that S6 can implement while transformers cannot.

In this work we compare the expressive power of S6 to those of transformers. We show that under certain assumptions which we justify empirically, a constant number of S6 layers are dense in the polynomial function space while transformers and non-selective SSMs are far less expressive. %
In addition to discussing expressivity, we also provide the first norm-based generalization bound for S6. Relevant related works are detailed in Appendix~\ref{sec:RelatedWorkGeneralization}.

\vspace{-5pt}
\section{Background}\label{sec:background}
\vspace{-3pt}
In this section we present the technical details required for the theoretical analysis and provide relevant 
notations.

\noindent{\textbf{{Notations\quad}}  Let $X \in \mathbb{R}^{D \times L}$ be an input sequence of length $L$ with dimension size $D$, denote the element in the $i$-th channel and position $j$ as $X_{ij}$. We denote the entire channel at a specific position $j$ and the entire sequence at a specific channel $i$ as $X_{*j}$ and $X_{i*}$, respectively. For simplicity, we denote a general single channel of $X$ without channel index by $x := (x_1, x_2, \cdots, x_L)$ such that $x_i \in \mathbb{R}$.

\noindent{\textbf{{Mamba\quad}} Given these notations, a Mamba block, which is built on top of the S6 layer
, is specified as follows:
{\small
\begin{equation} \nonumber
    X = \sigma(\text{Conv1D}(\text{Linear}(U)), \quad Z = \sigma(\text{Linear}(U)) 
    \vspace{-3pt}
\end{equation}
}
\vspace{-8pt}
{\small
\begin{equation} \label{eq:mamba1}
    Y = \text{S6}(X),\quad \hat{Y} = \text{Linear}(Y \otimes Z)
\end{equation}
}

Here, $U$ is the input to the Mamba block, and $X$ is the input to the S6 layers, $X,Z,Y,\hat{Y} \in \mathbb{R}^{L \times D}$, the linear layers operate independently for each sequence element, $\sigma$ represents SiLU activation function, and $\otimes$ denotes element-wise multiplication with the gate branch. 

\noindent{\textbf{S6 }}
An S6 layer is a recent variant of SSM. A standard diagonal SSM is parameterized by a diagonal transition matrix $A \in \mathbb{R}^{N \times N}$, input and output matrices $B,C \in \mathbb{R}^{N \times 1}$, and a timescale $\Delta \in \mathbb{R}$. An input scalar sequence $x$ is mapped to an output scalar sequence $y$ via the following recurrent rule:
\vspace{-5pt}
{\small
\begin{equation}\nonumber
\vspace{-3pt}
h_t =  \bar{A} h_{t-1} + \bar{B}x_t, \quad y_k = C h_t
\vspace{-5pt}
\end{equation}
}
{\small
\begin{equation}\label{eq:recRule}
\vspace{-3pt}
\bar{A}=f_A (A, \Delta), \quad  \bar{B}=f_B (B, \Delta)
\end{equation}
}
%
where $f_A,f_B$ are discretization functions, and the discrete system matrices are $\bar{A} \in \mathbb{R}^{N \times N}$ and $\bar{B} \in \mathbb{R}^{N \times 1}$. The recurrent rule in Eq.~\ref{eq:recRule} can be computed efficiently in parallel on modern hardware accelerators using work-efficient parallel scans~\cite{smith2022simplified} or a simple scalar convolution via FFTs~\cite{gu2021combining}. Note that Eq.~\ref{eq:recRule} is a map from $\mathbb{R}^L$ to $\mathbb{R}^L$, and to process $D$ channels, multiple independent instances are used.

Contrary to traditional SSMs, which utilize time-invariant system matrices and process each channel independently, S6 layers incorporate a data-dependent mechanism that is parameterized by $S_B, S_C \in \mathbb{R}^{N \times D}$, $A \in \mathbb{R}^{D \times N}$ and $S_{\Delta} \in \mathbb{R}^{1 \times D}$ to define the time-variant matrices as follows:
{\small
\begin{equation}\nonumber 
    \vspace{-3pt}
    B_t = S_B X_{*t}, \ C_t = S_C X_{*t}, \ \Delta_t = \text{softplus}(S_{\Delta} X_{*t})
    \vspace{-2pt}
\end{equation}
}
\vspace{-5pt}
{\small
\begin{equation}\label{eq:TimeVariantMatrices1} 
\bar{A}_t = \exp(\Delta_t A), \quad \bar{B}_t = \Delta_t B_t
\end{equation}
}
%
The resulting time-variant recurrent rule is:
{\small
\begin{equation}\label{eq:timeVaraintRecRule}
     \vspace{-2pt}
     h_t =  \bar{A}_t h_{t-1} + \bar{B}_t x_t, \quad y_k = C_t h_t
\end{equation}
}
Our analysis focuses on the regime of 'many-to-one', which deals with models that operate on sequences 
and produce a single output after processing the entire input sequence.
%

\vspace{-4pt}
\section{Theoretical Results}
\vspace{-3pt}
We begin by presenting our simplified model in Sec.~\ref{sec:simplify}, which forms the basis for our theory on the expressivity {\color{black}and generalization of S6 layers, discussed in} Sec.\ref{sec:expressivity} and Sec.~\ref{sec:generalization}, respectively. In the experiments section we demonstrate that the simplified model used in our exposition achieves comparable results to those of the original S6 layers, encouraging further exploration of the suggested simplification.
%

\vspace{-3pt}
\subsection{Model Simplifications\label{sec:simplify}} %
\vspace{-2pt}
The original S6 layer is parameterized by $S_{\Delta}$, $S_{B}$, $S_C$ and $A$, and it is defined in Eqs. \ref{eq:TimeVariantMatrices1} and \ref{eq:timeVaraintRecRule}.%

Our approach utilizes a simplified model described below:
{\small
\begin{equation*}
      \bar{B}_i = S_B x_i,\quad C_i = S_C x_i,\quad 
\Delta_i = p_1 (S_{\Delta} x_i),\quad \bar{A}_i =  p_2 ( \Delta_i A) 
\end{equation*}
}

where $p_1$ and $p_2$ represent polynomials that operate independently per element, for instance, a second-degree Taylor approximation for softplus and exponent accordingly. An equivalent model would be:
{\small
\begin{equation}\label{eq:simplifiedModel}
 \vspace{-5pt}
    C_i = S_C x_i, \quad \bar{B}_i = S_B x_i, \quad \bar{A}_i = p_2\left(p_1 \left(\frac{S_\Delta x_i}{\sqrt{D}}\right) A\right)
\end{equation}
}

in which $D$ is the width of the model, and $\frac{1}{\sqrt{D}}$ is a constant normalization factor. This model can be interpreted as state-space layer without discretization, with $\bar{A}$ being selective, and $p_1$ and $p_2$ are stabilizers designed to control the values of $A$, which must be positive. It is imperative to note that standard SSMs without discretization are both effective and simple~\cite{gupta2022simplifying}.


For simplicity, we define an additional polynomial $p_A(x) = p_2\left(p_1 \left(\frac{x}{\sqrt{D}}\right) A\right)$ which is parameterized by $A$ and ties $p_1$ and $p_2$. Hence, we can denote $\bar{A}_i = p_A(S_{\Delta} x_i)$.

Alternatively we can utilize the following simplified non-polynomial model:
{\small
\begin{equation}\nonumber
\vspace{-3pt}
\bar{B}_i = S_B x_i, \quad C_i = S_C x_i
\vspace{-2pt}
\end{equation}
}
\vspace{-5pt}
{\small
\begin{equation}\label{eq:model}
\Delta_i = \text{softplus}(S_{\Delta} x_i),\quad
\bar{A}_i =  \exp ( \Delta_i A)
\end{equation}
}
%
%
%
%
\vspace{-18pt}
\subsection{Expressivity\label{sec:expressivity}
}
\vspace{-3pt}
\input{Expressivity}

\vspace{-4pt}
\subsection{Generalization\label{sec:generalization}}
\vspace{-2pt}
\input{Generalizatoin}

\vspace{-7pt}
\section{Experiments}\label{sec:experiments}
\vspace{-4pt}
In this section, we extensively validate our theorems and assumptions through empirical analysis. First, in Sec.\ref{sec:modelJustification},
we demonstrate that our simplified variant of the Mamba layer achieves performance comparable to the original Mamba layer when incorporated into standard settings and deep networks, thereby justifying the exploration of this variant. Next, in Sec.\ref{sec:LearningPoly},
we validate our theory on expressiveness by showing that self-attention struggles to learn high-degree multivariate polynomials, which S6 can model effectively. 

\vspace{-5pt}
\subsection{Model Justification\label{sec:modelJustification}}
\vspace{-3pt}
Our theoretical study employs the simplified S6 variant described in Eq.~\ref{eq:simplifiedModel}. We conduct experiments in both the NLP and vision domains, evaluating this variant when integrated into the Mamba backbone, with the goal of showing that it performs similarly to the original S6 layer.

{\noindent\textbf{NLP\quad}}
In NLP, we trained variants of our simplified S6 layer within Mamba backbones on the Wikitext-103 dataset using a self-supervised scheme for \textit{Next Token Prediction}. Our models feature 12 layers with a hidden dimension size of 386 and were trained with a context length of 1024 tokens. The final results are detailed in right panel of Tab.~\ref{tab:empricalModelJustifiction} and in the right panel of Fig.~\ref{fig:modelJustifications}, which illustrates the evolution of test-perplexity across epochs. Evidently, our simplified S6 variant performs well in the NLP domain, with only a slight reduction in perplexity with respect to the original model. Specifically, the polynomial variant achieved a perplexity score of 26.42, 0.69 points lower than its original baseline score of 25.73. In contrast to the polynomial variant, the other simplified variants that employ $\bar{B}_i = B_i$ achieve a slightly lower perplexity 
compared to the baseline, highlighting the significance of this aspect in the architecture.

{\noindent\textbf{Vision\quad}} %
Image classification experiments are conducted on the ImageNet-100 benchmark. We built upon the Vision-Mamba (ViM) architecture~\cite{mambaViT1}, replacing the S6 layers with our simplified variant while maintaining the same training procedures and hyper-parameters. The left panel of Tab.~\ref{tab:empricalModelJustifiction} presents the results: the simplified variant from Eq.~\ref{eq:simplifiedModel} achieves a 
accuracy of 78.62\%, which is 2.4\% lower than the original model which achieve a score of $81.02$. For reference, we include the results of DeiT~\cite{touvron2021training}, which achieved a top-1 accuracy of 78.21\% for the same model size
~\cite{baron2024a}. %
%
%

\begin{table*}[]
    \centering
    \small
    \vspace{-11pt}
    \caption{Ablations of our simplified S6 variants, with vision tasks on the left and NLP on the right. 'S' for simplified variants. Results for Transformer models are provided as a reference point.}
    \smallskip
    \label{tab:empricalModelJustifiction}
    \begin{tabular*}{0.48\linewidth}{@{\extracolsep{\fill}}lcc}
        \toprule
        Model & Top-1 & \# Parameters \\
        \midrule
        ViM (baseline) & 81.02 & 6.2M  \\
        ViM (S., $\bar{B}_i = B_i$) & 79.36 & 6.2M\\
        ViM (S., $\bar{A}_i = p_A(S_{\Delta}(x_i))$) & 80.28 & 6.2M \\
        ViM (S., Eq.~\ref{eq:simplifiedModel}) & 78.62 & 6.2M  \\
        Transformer (DeiT) & 78.21 & 6.2M  \\
        \bottomrule
    \end{tabular*}
    \hfill
    \begin{tabular*}{0.49\linewidth}{@{\extracolsep{\fill}}lcc}
        \toprule
        Model & PPL & \# Parameters \\
        \midrule
        Mamba (baseline) & 25.73 & 30.1M  \\
        Mamba (S., $\bar{B}_i = B_i$) & 29.49 & 30.1M \\
        Mamba (S., $\bar{A}_i = p_A(S_{\Delta}(x_i))$) & 26.42 & 30.1M \\
        Mamba (S., Eq.~\ref{eq:simplifiedModel}) & 31.12 & 30.1M  \\
        Transformer & 28.31 & 31.4M  \\
        \bottomrule
    \end{tabular*}
    \vspace{-7pt}
\end{table*}

To identify which aspects of our simplification most significantly impact performance, we compare two additional variations. First, we use a polynomial S6 variant without omitting the discretization. Second, we run a vanilla non-polynomial model where be $\bar{B}_i = B_i$. Our empirical analysis reveals that the polynomial model performs remarkably well, achieving an accuracy score of 80.28, just 0.74 points below the original model and 0.92 points above the non-polynomial simplified model. To provide a comprehensive view, the training curves are presented in left panel of Fig.~\ref{fig:modelJustifications}, which also empirically analyzes several variants of the simplified model compared to the baseline. The full set of hyper-parameters can be found in the Appendix at Tab.~\ref{tab:Vsionhyperpams}.

\begin{figure}[t]
\vspace{-6pt}
\centering
\begin{tabular}{cc}
    \includegraphics[width=0.225\textwidth]{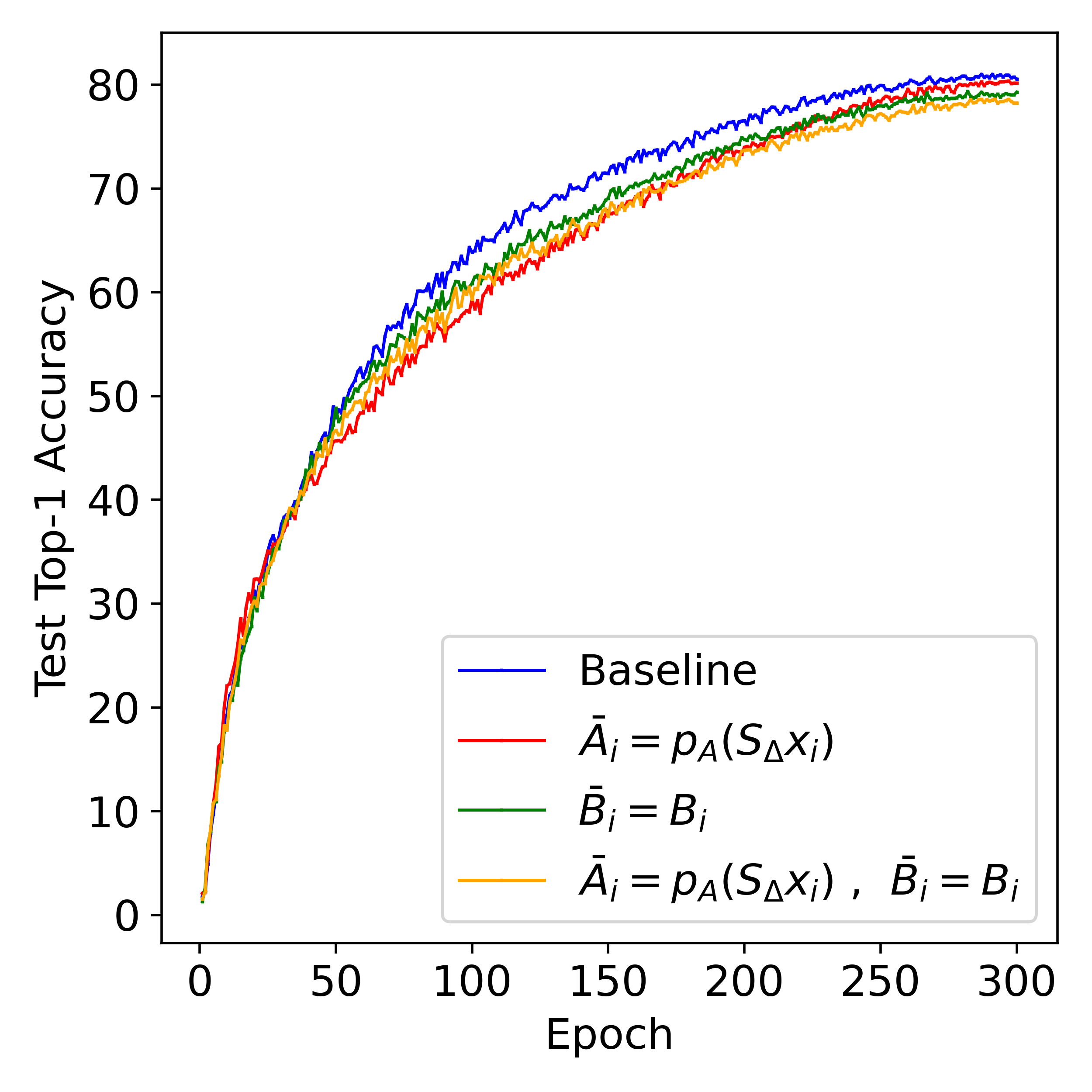} &
    \includegraphics[width=0.225\textwidth]{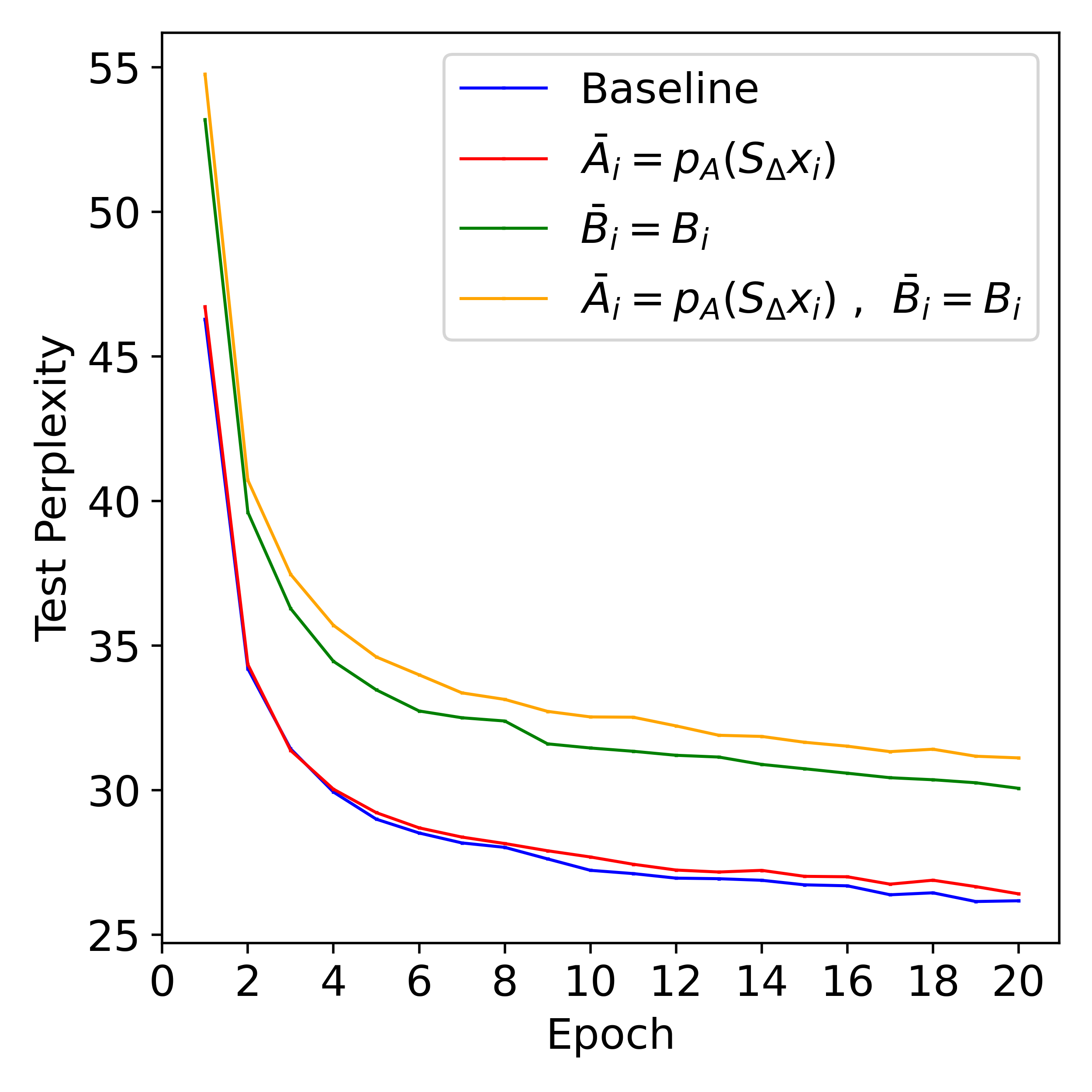} \\
\end{tabular}
\vspace{-8pt}
\caption{\textbf{Model justifications \& ablations: }In the \textbf{left} panel, we present the top-1 accuracy score
for image classification via the ImageNet-100 benchmark, while the \textbf{right} panel displays the perplexity score
for language modeling using the WikiText-103. The y-axis represents the model's score across different epochs. In both figures, the blue curve represents the baseline, the yellow curve corresponds to Eq.\ref{eq:simplifiedModel}, the green curve illustrates Eq.\ref{eq:model}, and the red curve depicts the polynomial variant using standard discretization.
}
\vspace{-9pt}
\label{fig:modelJustifications}
\end{figure}

\vspace{-4pt}
\subsection{Learning Polynomials\label{sec:LearningPoly}}
\vspace{-2pt}

In this section, we empirically validate our theoretical claims concerning the expressiveness of S6 and self-attention layers from Thms.~\ref{theorem:exprrsFull} and ~\ref{theorem:AnyPolywithMambas}. Since our analysis of the expressiveness gap between those layers relies on their characterization via multivariate polynomials, we focus on learning such functions over synthetic data. To isolate factors other than expressiveness, we employ a control setup with small NNs, comprising up to four layers with narrow widths (2 or 4 channels) and an additional output linear projection head. For each architecture, we used the standard implementations: (i) self-attention with softmax and positional encoding, and (ii) the original S6 architecture~\cite{gu2023mamba} with and without PE. The experiments examine the clean implementation of these components, without additional elements such as Conv1D, activations, or FFNs. We conduct experiments on two tasks: classification and regression. 

{\noindent\textbf{Classification}\quad} Our dataset consists of binary random sequences 
of length $L=20$. The labels are uniformly distributed between 0 and $L$, determined using the ``Count in Row'' function~\cite{ali2024hidden}, defined as follows:
    \begin{definition}
        The count in row problem: Given a binary sequence $x_1, x_2, \ldots x_L \in  \{0,1\}^L$ such that $x_i \in \{0,1\}$ 
        , the ``count in row'' function $f$ is defined to produce an output sequence $y_1, y_2, \ldots, y_L$, where each $y_i = f(x_1, .. , x_i)$ is determined based on the contiguous subsequence of 1s to which $x_i$ belongs. Formally: 
        \vspace{-7pt}
        {\small
        \begin{equation}
        y_i = \max_{0\leq j \leq i} \Big{(} \{ i-j+1 \mid \prod_{k=j}^i [x_k > 0] = 1\} \cup \{ 0\} \Big{)}
        \end{equation}
        }
        %
        where $[x_k > 0]$ is the Iverson bracket, equaling 1 if $x_k> 0$ and 0 otherwise.
    \end{definition}
    \vspace{-5pt}
    
The top part in Tab.~\ref{tab:EmpricalExpressClassification} presents the results. Remarkably, even a single layer of selective SSMs, both with and without PE, outperforms attention models with double the number of layers and channels, all while utilizing significantly fewer parameters, as suggested by Thm~\ref{theorem:exprrsFull}.


\begin{table}[t]
\vspace{-8pt}
\caption{\small\textbf{Learning multivariate polynomials over synthetic data.} Classification results are presented in the top, while regression results are displayed on the bottom. Best results for each model depth in bold. 
'$D$' for the number of channels.}
\smallskip
\small
\label{tab:EmpricalExpressClassification}
\centering
\begin{tabular*}{\linewidth}{@{\extracolsep{\fill}}lccc}
\toprule
\multicolumn{4}{c}{Classification}\\
\toprule
 Model &   \# Layers &  Accuracy  &  \# Parameters \\
\midrule
S6 w/ PE $(D=2)$ & 1  & 83.1 & 35 \\
S6 w/o PE $(D=2)$ & 1 & \textbf{84.8} & 35 \\
S6 w/ PE $(D=2)$ & 2 & 93.4 & 63 \\
S6 w/o PE $(D=2)$ & 2 & \textbf{97.1} & 63 \\
\midrule
Self-Attention$(D=2)$ & 1 & 32.9 & 29 \\
Self-Attention$(D=2)$ & 2 & 41.1 & 51 \\
Self-Attention$(D=2)$ & 4 & 44.8 & 95 \\
Self-Attention$(D=4)$ & 1 & 36.2 & 176 \\
Self-Attention$(D=4)$ & 2 & 44.2 & 244 \\
Self-Attention$(D=4)$ & 4 & 55.8 & 380 \\
\bottomrule
\end{tabular*}
\hfill
\label{tab:CelebA}
\centering
\begin{tabular*}{\linewidth}{@{\extracolsep{\fill}}llcc}
\toprule
\multicolumn{4}{c}{Regression}\\
\toprule
 Model &   D  &  MSE  & \# Parameters \\
\midrule
S6 w/ PE & 4 & 12.67  & 101 \\
S6 w/o PE & 4 & \textbf{11.81} & 101   \\
S6 w/ PE & 6 & 12.45  & 157 \\
S6 w/o PE & 6 & \textbf{11.04} & 157   \\
S6 w/ PE & 8 & 12.17  & 377 \\
S6 w/o PE & 8 & \textbf{9.057} & 377 \\
\midrule
Self-Attention & 4 & 19.22   & 81\\
Self-Attention & 6 & 19.10   & 157\\
Self-Attention & 8 & 19.048  & 257\\
\toprule
\end{tabular*}
\vspace{-19pt}
\end{table}

{\noindent\textbf{Regression}\quad} We synthetically construct the dataset \( S = \{(x_i, y_i)\}_{i = 1}^m \)
by first randomly selecting a polynomial denoted by P. For each example in the dataset, we then generate $x$ values uniformly at random and compute the corresponding labels using this P.
{\small
\begin{equation}
\vspace{-3pt}
   c_i \sim \mathbb{U}([-2,2]), \text{  } p_{i,j}\sim \mathbb{U}([ L ]), \text{  } x_j \sim \mathbb{U}([0.1,2])
   \vspace{-4pt}
\end{equation}
}
\vspace{-6pt}
{\small
\begin{equation}\label{eq:randomPoly}
\vspace{-4pt}
   y = P(x) = \sum_{i = 1}^3 c_i \Pi_{j=1}^L {x_j}^{p_{i,j}}
\end{equation}
}

Our models consist of a single layer with either 4 or 8 channels, processing sequences of length $L=5$. As demonstrated in the bottom part of Tab.~\ref{tab:EmpricalExpressClassification}, both S6 variants, with and without PE, significantly outperform traditional self-attention layers across both model sizes. For example, while a self-attention model with 8 channels obtains an MSE score of 19.05, all S6 variants achieve an MSE below 12.67. These experiments demonstrate that at least in controlled environments with small models, S6 layers outperform traditional self-attention layers in approximating polynomials where the total multivariate degree exceeds the sequence length.
%
%
%
%
%
%
%

\vspace{-6pt}
\section{Discussion}
\vspace{-2pt}
To understand the implications of our theory, we first explain why analyzing Softmax-free attention is realistic. Then, we discuss the consequences for standard attention models.

\noindent\textbf{Transformers Without Softmax\quad} The softmax function is primarily associated with optimization and stability, as it normalizes attention scores to the [0, 1] range and prevents numerical instabilities. However, transformer variants without softmax have proven effective in several domains, including reducing latency~\cite{hua2022transformer,lu2021soft,ramapuram2024theory} and in applications such as vision~\cite{wortsman2023replacing}, NLP~\cite{ma2022mega}, and other areas~\cite{zimerman2023converting}. {\color{black}Additionally, these models have recently become even more practical, as researchers have successfully scaled linear attention far beyond 7B parameters~\cite{li2025minimax}, enabling LLMs to extend their context window to 4 million tokens while matching the performance of GPT-4o and Claude-3.5 Sonnet. Several additional linear attention-based LLMs were presented in~\cite{shen2024scaling, qin2023transnormerllm, sun2023retentive}.}
Since these models achieve near-SoTA, focusing on a softmax-free attention model is well justified.

\noindent\textbf{Transformers With Softmax\quad} While the softmax function can theoretically be expressed as an infinite-degree polynomial, we provide careful considerations. The softmax function involves both exponentiation and proportional normalization. The former can be well approximated using low-degree polynomials~\cite{zhang2024secure}, while the latter primarily serves to normalize the scores. We refine our assumption by analyzing transformers that apply exponentiation to each attention score, assuming this can be approximated by a polynomial of degree P. The resulting model expresses higher-degree polynomials within each layer, but it remains based on pairwise interactions via Key, Query, and Values, leading to an maximal polynomial degree of $3P+1$, independent of the sequence length $L$. This supports the validity of our argument in more common regimes.

%
\noindent\textbf{Interpretation and Intuition}\quad}
Our characterization of S6 layers through the lens of 
polynomials offers a novel perspective on the semantic capabilities of Mamba. Specifically, we extend the concept of polynomial degree to quantify the number of tokens involved in each interaction within a layer of an model. For instance, low-degree polynomials correspond to interactions involving only a few tokens, while high-degree polynomials represent dependencies spanning many tokens. This analysis highlights the unique strength of S6 layers in modeling continuous, multi-token interactions, such as counting and recurrent operations. In contrast, transformers, 
are naturally biased toward sparser and more fragmented representations such as induction heads. 

This perspective can also shed light on the remarkable performance of hybrid models that combine modern RNNs and attention by leveraging their complementary strengths~\cite{lieber2024jamba,de2024griffin}. While a formal characterization of their trade-offs is yet to be established, our analysis suggests that S6 and attention capture fundamentally distinct types of interactions, characterized by the number of tokens involved in each interaction.

\vspace{-8pt}
\section{Conclusions}
\vspace{-3pt}
This study explores the expressivity of Mamba models. By reducing the S6 layer to a polynomial form and composing an associated theory, we have established a novel connection between S6 layers and high-degree multivariate polynomials. This connection enables us to identify the expressivity gap between S6 layers and attention mechanisms comprehensively. {\color{black}We show that although the S6 layer has better theoretical expressivity than linear attention for long sequences, this does not negatively impact generalization. We provide a length-agnostic generalization bound to support this result, allowing us to conclude that the S6 layer has superior theoretical properties compared to linear attention for long-range tasks.} 
Finally, the limitations of our work are discussed in Appendix~\ref{sec:limitations}.
%
%
%
\newpage


\section{Acknowledgments}
This work was supported by a grant from the Tel Aviv University Center for AI
and Data Science (TAD) and the Ministry of
Innovation, Science \& Technology ,Israel (1001576154) and the Michael J. Fox
Foundation (MJFF-022407). This research was also supported by the European Research Council (ERC) under the European Unions Horizon 2020 research and innovation programme (grant ERC HOLI 819080).

\nocite{langley00}

\bibliography{example_paper}
\bibliographystyle{icml2025}

\newpage
\appendix
\onecolumn

\input{appendix_experiments}

\input{appendix_proofs}

\input{appendix_generalization}

\section{Additional Related Work on Generalization\label{sec:RelatedWorkGeneralization}}
Explaining the performance of overparameterized deep neural networks (DNNs) on test data remains a major challenge in deep learning theory. Traditional tools like the PAC-learning framework and VC-dimension often provide vacuous bounds when the number of parameters greatly exceeds the number of data points. To address this, many studies conduct architecture-specific analyses. For instance, Allen et al. analyze the dynamics of stochastic gradient descent on RNNs with ReLU activations, offering optimization and generalization guarantees~\citep{allen2019can}. Other works have explored the generalization of RNNs for unseen data and longer sequences under various assumptions~\citep{cohen2022implicit, emami2021implicit, cohen2022learning, hardt2018gradient}.

There are relatively few results that address modern architectures such as S6 layers. A recent contribution proposes a bound for standard SSMs~\citep{liu2024generalization}, but it does not extend to Selective SSMs. To address this gap, we propose a new bound specifically for Selective SSMs.

\section{Limitations\label{sec:limitations}}
In this paper, we provide a new perspective on the expressivity gap between S6 layers and self-attention, the core layers of Mamba models and transformers. While our analysis leverages multivariate polynomial degrees as measures of expressiveness and offers insightful results, it does not formally connect these measures to widely-used expressivity metrics in the literature, such as Rademacher complexity or VC dimension. Establishing such connections remains an open challenge.

Additionally, our work focuses on a simplified architecture, omitting several components of the original models. While we empirically justify these simplifications and highlight the opportunity to use simpler models by identifying the key components responsible for the performance gap, analyzing the full complexity of softmax-based Transformer and Mamba architectures is an important direction for future research. 

Finally, although expressiveness is a critical property to explore, LLMs with billions of parameters involve additional factors that influence their capacity and performance. These include optimization challenges, gradient behavior, implicit biases, and training stability. Addressing these aspects is beyond the scope of this study, which is centered on a theoretical characterization of expressiveness. We believe these topics represent promising avenues for future work.

\end{document}

%% file: math_commands.tex
\usepackage{amsmath,amsfonts,bm}

\def\eqref#1{equation~\ref{#1}}

\def\1{\bm{1}}

\DeclareMathAlphabet{\mathsfit}{\encodingdefault}{\sfdefault}{m}{sl}
\SetMathAlphabet{\mathsfit}{bold}{\encodingdefault}{\sfdefault}{bx}{n}

\newcommand{\E}{\mathbb{E}}

\newcommand{\R}{\mathbb{R}}

\def\fr#1#2{{\textstyle\frac{#1}{#2}}} 
\usepackage[utf8]{inputenc} 
\usepackage[T1]{fontenc}    
\usepackage{url}            
\usepackage{booktabs}       
\usepackage{amsfonts}       
\usepackage{nicefrac}       
\usepackage{microtype}      
\usepackage{xcolor}         
\usepackage{amsmath}
\usepackage{amsthm}
\usepackage{amssymb}
\usepackage{mathrsfs}
\newtheorem{lemma}{Lemma}
\newtheorem{definition}{Definition}
 \newtheorem{theorem}{Theorem}

\usepackage{graphicx}
\usepackage{thmtools} 
\usepackage{thm-restate}

\newcommand{\err}{\textnormal{err}}

\newcommand{\bI}{\mathbb{I}}

\usepackage{xcolor}

\usepackage{mathrsfs}
\usepackage[scr=esstix]{mathalfa}

%% file: Expressivity.tex
We identify an expressivity gap between S6 and attention via multivariate polynomials. 
This gap is delineated through Thm.~\ref{theorem:exprrsFull} and Thm.~\ref{theorem:AnyPolywithMambas}. The former establishes that attention models require $O(\log L)$ layers to represent $L$-degree multivariate polynomials, whereas Mamba models can express such polynomials within a single layer. Furthermore, Thm.~\ref{theorem:AnyPolywithMambas} extends this finding by demonstrating that the expressiveness gap is not limited to anecdotal or edge-case polynomials. Rather, it encompasses a broad spectrum of polynomial functions.


\begin{theorem}\label{theorem:exprrsFull}(informal)
Consider an S6 layer and single Transformer layer, both with hidden dimension $N$. 
For input sequences of length \( L \geq 3 \), a single layer of Mamba is logarithmically more expressively efficient in depth compared to a single causal {\color{black}linear} self-attention layer with a single head and polynomial activations instead of softmax. 
\end{theorem}

We consider this theorem relatively surprising, as transformers are considered highly expressive models, in contrast to state-space layers, which are traditionally constrained. The proof follows from Lemma~\ref{lemma:dir2} and Lemma~\ref{lemma:dir1} which are presented next.


\begin{lemma}\label{lemma:dir2}
There exists a function \( f : \mathbb{R}^L \rightarrow \mathbb{R} \) that can be implemented by one channel of S6 such that a single {\color{black}linear} attention head would require at least  \( O(\log L) \) layers to express $f$.
\end{lemma}

\begin{lemma}\label{lemma:dir1}
Any function that can be expressed by a single {\color{black}causal linear} attention head can 
be expressed by a single channel of S6.
\end{lemma}

For the complete details and proof of Lemma~\ref{lemma:dir1}, we refer the reader to the appendix (Lemma \ref{lemma:dir1Appendix}). As for Lemma~\ref{lemma:dir2}, we present here a proof sketch for a simplified version (see Lemma~\ref{lemma:dir2appendix}
in the appendix for the complete proof). In this simplified case, our proof focuses exclusively on linear-attention
. Additionally, we consider models that process scalar sequences. 

\begin{proof}[Proof of Lemma~\ref{lemma:dir2} (without Softmax
)]
Our proof relies on the characterization of the hypothesis classes that are realizable through S6 and self-attention via {\underline{multivariate polynomials}}. We start by presenting this formulation for S6:

\smallskip
\noindent\textbf{Single S6 Layer as Multivariate Polynomials\quad}
One channel of the variant of the S6 layer we consider is described in detail in Eq.~\ref{eq:simplifiedModel}.
%
%
%
Since we are interested in identifying the minimal polynomial degree required to characterize models that employ Eq.~\ref{eq:simplifiedModel}, we can assume that $P_A$ is linear. Thus, we can incorporate the coefficients of $P_A$ into $S_{\Delta}$, namely, $\bar{A}_t = S_{\Delta} x_t$. Consequently, Eq.\ref{eq:simplifiedModel} can be unrolled as follows:
{\small
\begin{equation}\label{eq:unrolledChanSl}
    y_{t}  
    = S_C S_B  \sum_{j=1}^t {S_{\Delta}}^{t-j-1} \Big{(} {\Pi_{k=j+1}^t x_k  }\Big{)} x_t {x_j}^2 
\end{equation}
}

Hence, we can characterize the last output $y_L$ as a multivariate polynomial with at least $L$ monomials, and a maximal degree of at least $L+3$. Denote $c_j = S_C S_B S_{\Delta}^{L-j-1}$,
\vspace{-5pt}
\begin{equation}\label{eq:SimpleMamba1Chan}
    y_L = \sum_{j=1}^L c_j \Big{(} \Pi_{k=j+1}^{t-1} x_k \Big{)} {x_t}^2 {x_j}^2 
\end{equation}
\noindent\textbf{Attention as Multivariate Polynomials\quad}
Given the parameters of the layers $W_Q, W_K, W_V \in \mathbb{R}$, the last element in the output sequence can be computed by:
\vspace{-5pt}
{\small
\begin{equation}
y=\frac{(x W_{Q}) (x  W_{K})^T}{\sqrt{d_k}} (x W_{V})
\end{equation}
}
\vspace{-5pt}
{\small
\begin{equation}
\small
y_L = W_Q W_K W_V \sum_{j=1}^L  {x_j}^2 x_L = \sum_{j=1}^L c_j{x_j}^2 x_L 
\vspace{-3pt}
\end{equation}
}

where $c_j = W_Q W_K W_V$. Hence, we can characterize the last output element $y_L$ of a self-attention layer by a 3-degree multivariate polynomial with $L$ monomials.

\noindent\textbf{$N$-Stacked Attention Layers as Multivariate Polynomials\quad}
In light of the characterization of one head of a self-attention layer, we can now proceed to establish the characterization of $N$-stacked self-attention layers. Since the composition of multivariate polynomials results in another multivariate polynomial, where the maximal degree of the resulting polynomial amounts to the product of the maximum degrees of the composed polynomials, we can prove by induction that $N$-Stacked layers can be represented by a polynomial with a maximal degree of $3^N$.

\noindent\textbf{Identify and Refine the Expressivity Gap\quad} The description of one head of a self-attention layer and one channel of an S6 via multivariate polynomials reveals the expressiveness gap in terms of the maximal degree, which is 3 for self-attention and $L+3$ for S6, as depicted in 
Fig.~\ref{fig:mainfig}. Thus, our formulation presents a broad family of functions that can be implemented by selective SSMs but cannot be realized by a single self-attention head. Furthermore, when processing a sequence of length $L$, it is clear that in order to express a function realized by S6 using $N$-stacked self-attention layers, $O(\log L)$ stacked layers are required.
\end{proof}

\noindent\textbf{Single SSM as Multivariate Polynomials\quad} For completeness, we formalized a single standard SSM (not S6) via multivariate polynomials. Since SSMs can be represented as a single non-circular convolution between the input \(x\) and a kernel \(\bar{K} = (C\bar{B}, C\bar{A}\bar{B}, \ldots, C\bar{A}^{L-1}\bar{B})\), it is evident that the last output element of the SSM can be expressed by a 1-degree multivariate polynomial consisting of \(L\) monomials.
%
%

\noindent\textbf{Sharpening the Expressivity Gap.} 
Thm~\ref{theorem:exprrsFull} establishes the existence of an expressivity gap. However, it remains unclear how many polynomials are encompassed within this gap, and whether they constitute a significant portion of the function class or are merely anecdotal instances. To refine our separation results 
, we now quantify the expressiveness gap by analyzing the number of layers required to represent the entire class of $L$-degree multivariate polynomials. As established previously, attention-based models necessitate $O(\log L)$ layers. In contrast, the following theorem demonstrate that a Mamba model with just 4-stacked layers and a sufficiently large number of channels can represent \textit{any multivariate polynomial of arbitrary degree}, thereby highlighting the significant expressiveness gap, which constitute a significant portion of the class of $L$-degree polynomials.

\begin{theorem}\label{theorem:AnyPolywithMambas}
Given an input scalar sequence $x \in \mathbb{R}^L $, a model with four stacked Mamba layers, a sufficiently large number of channels, learnable PE, sufficient padding and a linear encoder layer at the first layer can express any  multivariate polynomial with bounded degree of $x$.
\end{theorem}

For simplicity, we assume that the state size $N = 1$ and the SiLU activations in Mamba are removed. We justify these decisions as they only restrict our model. The proof is presented in the A
appendix, and it follows a technical construction that demonstrates how to express a general polynomial using Mamba. The core argument is built on the following lemma, which is 
visualized in Fig.~\ref{fig:exprresTheorem}:

\begin{figure*}[]
    \centering
    \vspace{-4pt}
    \includegraphics[width=0.92\linewidth]{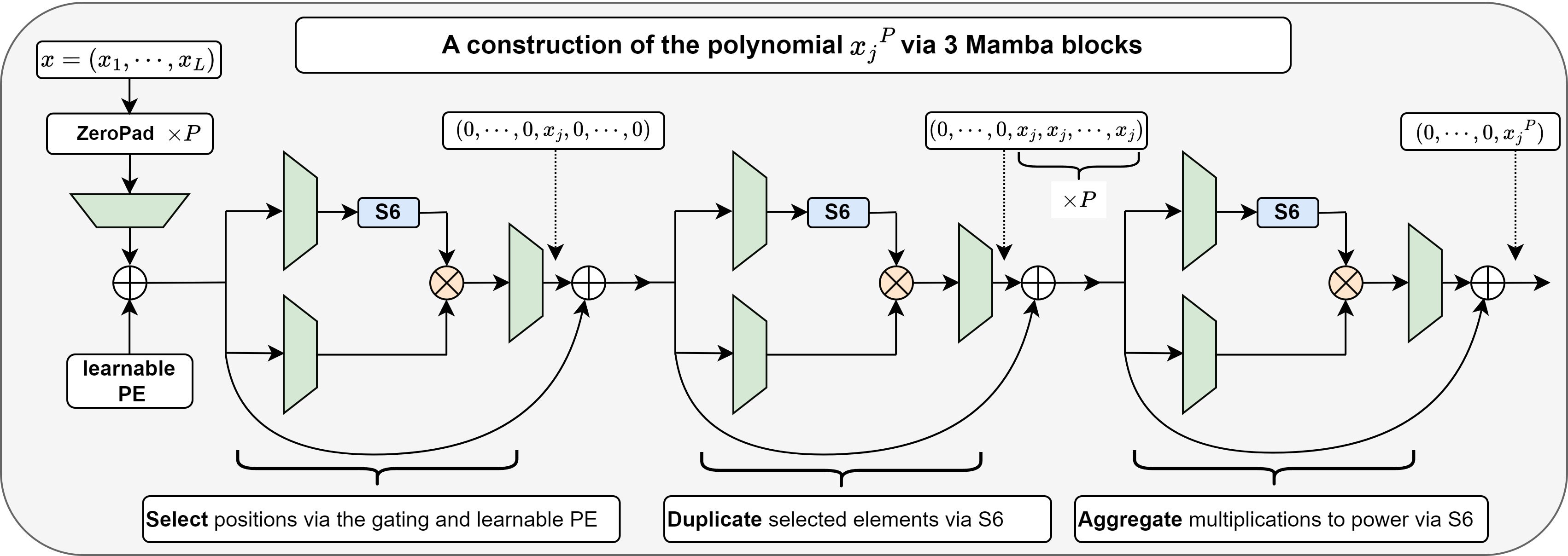}
    \vspace{-7pt}
    \caption{Visualization of 3-stacked Mamba layers expressing monomials of a univariate polynomial, as formulated in Lemma~\ref{lemma:3layerMambaExpresivity}. To simplify the visualization, the Conv1D layer has been omitted.}
    \label{fig:exprresTheorem}
     \vspace{-10pt}
\end{figure*} 

\begin{lemma}\label{lemma:3layerMambaExpresivity}
Given an input scalar sequence $x \in \mathbb{R}^L$, for any $j$, a model $M$ with 3-stacked Mamba layers, a sufficiently large number of channels, learnable PE, and a linear encoder in the first layer can express any monomial of a univariate polynomial in $x_j$. Specifically, for any constants $c \in \mathbb{R}$ and $P \in \mathbb{R}$, there exists a configuration of $M$ such that the output of the $k$-th channel $M(x)_k = c \cdot x_j^P$ for $k > P + j$.
\end{lemma}
For a detailed proof of Lemma~\ref{lemma:3layerMambaExpresivity}, we refer the reader to the appendix. Here, we provide an intuitive explanation of the proof, which hinges on the following two key capability of the Mamba architecture:

{\noindent\textbf{(i) Per-position selection:}} By utilizing Mamba's auxiliary components, including the gating branch, linear layers, and learnable PE, each Mamba layer can isolate a specific channel $k$ at a given position $j$. Notably, the output of the first Mamba block can effectively filter out all unnecessary positions, producing a sequence mask with zeros at every position except $i = j$, which contains $x_j$ at position $j$. This is done by setting the S6 to the identity function ($\bar{A}=0, \bar{B}=\bar{C}=1)$, ensuring $x_j$ is not modified. Additionally, mask the other positions achieved by set the parameters of a learnable PE at one of the channels to the indicator function $\mathbb{I}_{=j}$, which is 1 only when focusing on the $j$-th position, and ensuring this PE passed into the gate branch through the linear layers.

{\noindent\textbf{(ii) Express powers by aggregate multiplications of duplicated elements:}} Given an input $x = (0, \cdots, x_j, 0, \cdots)$, which can be obtained from the first layer, the second Mamba block can duplicate the value of $x_j$ exactly $P-2$ times. This duplication holds even if $P > L$, thanks to the ZeroPad component. The duplication process is achieved by setting the linear layers to identity mappings and utilizing a degenerate single SSM channel where the system matrices always equal $\bar{A},\bar{B},\bar{C}=1$. Therefore, if the $k$-th channel receives an input sequence $x = (0, \cdots, x_j, 0, \cdots)$, it will output $x = (0, \cdots, x_j, \cdots, x_j)$.
Through the gating mechanism, learnable PE (which can pass through skip-connections to the subsequent layers), and biases in the linear layer, the entire block can then produce a filtered output $z = (1, \cdots, 1, x_j, \cdots, x_j, 1, \cdots, 1)$, ensuring that there are exactly $P-2$ copies of $x_j$.
Next, the S6 layer in the third block can multiply the sequence elements in $z$ via the operation described in Eq.~\ref{eq:SimpleMamba1Chan}, which include the term $\Pi_{k=j+1}^t \bar{A}_k$. This term yields an output sequence with the values $(1, \cdots, 1, x_j^3, \cdots, x_j^P, \cdots)$. To specifically isolate ${x_j}^P$, we begin by generating all unwanted elements by applying the same SSM to the sequence $z$, introducing an additional zero at the initial occurrence of $x_j$ denoted by $z'$. We then subtract the outputs from these two SSM channels at the final linear layer of the block. This subtraction yields a telescoping series $SSM(z) - SSM(z')$:
\vspace{-2pt}
{\small
\begin{equation}
\sum_{j=1}^L c_j \Big{(} \Pi_{k=j+1}^{t-1} x_k \Big{)} {x_t}^2 {x_j}^2 - \sum_{j=2}^L c_j \Big{(} \Pi_{k=j+1}^{t-1} x_k \Big{)} {x_t}^2 {x_j}^2
\vspace{-3pt}
\nonumber
\end{equation}
\vspace{-3pt}
}
that effectively eliminates any term except for ${x_j}^P$. 


Finally, it is crucial to highlight that incorporating auxiliary components into the Transformer, such as learnable PE and gating, does not mitigate the logarithmic increase in depth required with L. This limitation arises because the expressible degree within each block remains unchanged, thereby leading to the observed asymptotic behavior.

%% file: Generalizatoin.tex
Our theoretical analysis in Sec.~\ref{sec:expressivity} demonstrates the superior expressive power of S6 compared to linear attention, particularly in terms of multivariate polynomial degree and long-range processing capabilities. Furthermore, Thm.~\ref{theorem:AnyPolywithMambas} and Lemma~\ref{lemma:dir2} establishes that the hypothesis class associated with Mamba models with just four layers is significantly larger than that of transformers, even with greater depth. While increased expressivity can often come at the cost of generalization, in this section, we show that S6 layers do not suffer from this trade-off. To do so, we %
provide a generalization bound for simplified S6 layers, as {\color{black}defined in Eq.~\ref{eq:modelAppendix}, which is a generalization of Eq.~\ref{eq:model}}. Our analysis is based on a classifier \( f : \mathbb{R}^{D \times L} \rightarrow \mathbb{R}^{\mathcal{C}} \), parameterized by  \( (A^{}, S_B^{}, S_C^{}, S_{\Delta}^{}) \).
The classifier \( f \) for a specific class \( c \in [\mathcal{C}] \) denoted by $f^c(X_{*1},...,X_{*L})$ is computed as:
%
\vspace{-3pt}
{\small
\begin{equation}
\sum_{d=1}^D W_{c,d} \left(S_C^{} X^{}_{*L}\right)^T \sum_{i=1}^L \left(\prod_{k=i+1}^L \bar{A}^{}_{dk} \right) S_B^{} X^{}_{*i} X^{}_{di}
\end{equation}
}
%
where $\mathcal{C}$ is the number of classes, $ \bar{A}^{}_{dk} = \exp(\Delta_k A_d)$ ($A_d$ is the $d$th row of $A$) as defined in Eq.~\ref{eq:TimeVariantMatrices1} and $W \in \mathbb{R}^{\mathcal{C} \times D}$ represents a linear projection from the output to the number of classes.
We denote the parameters of a classifier by 
$w  = \left(A, S_B, S_C, S_{\Delta}, W\right)$
and the corresponding function induced by a specific instance of $w$ by $f_w$.
The class of functions taking on different parameter instances $w$, is denoted by $\mathcal{F}$.
As customary in the derivation of Rademacher complexity based bounds (e.g. \citep{golowich2018size}), our analysis takes into account the (different) norms of the parameters, for which we 
denote:
\vspace{-3pt}
{\small
\begin{equation}\nonumber
\begin{aligned}
    \rho_W(w)&=||W||_{F},\; \rho_A^{}(w) = \|A^{}\|_{\max},\; \rho_B^{}(w) = \|S_B^{}\|_{2,\infty}, \\ \rho_C^{}(w) &= \|S_C^{(h)}\|_{F},\;
    \rho_{\Delta}^{}(w) = \|S_{\Delta}^{}\|_{2}, \\ \Gamma(w) &=  \rho_W(w) \cdot \rho_A^{}(w) \cdot \rho_B^{}(w) \cdot \rho_C^{}(w) \cdot \rho_{\Delta}^{}(w)
\end{aligned}
\end{equation}
}
Equipped with these notations, we are now ready to state our main generalization bound.
 \begin{theorem}\label{theorem:genbound}
Let $P$ be a distribution over $\mathbb{R}^{D \times L} \times [C]$ and $\delta > 0$. Let $S = \{( X^{}_{(j)},y_{(j)})\}^{m}_{j=1}$ be a dataset of i.i.d. samples selected from $P$, where each \(X_{(j)} = (X_{(j)_{*1}}, \ldots, X_{(j)_{*L}}) \in \mathbb{R}^{D \times L}\). Assume that $\forall j \in [m]: ||X_{(j)}||_{\max} \leq 1$. Additionally, suppose $\forall k \in [L], d \in [D]:||\bar{A}^{}_{dk}||_{\max} < K < 1$. Then, with probability at least $1-\delta$ over the selection of $S$, for any $f_w \in \mathcal{F}$, 
{\small
\begin{equation*}
\begin{aligned}
&\err_P(f_w) - \fr{1}{m}\sum^{m}_{j=1}\bI[\max_{c \neq c'}(f^c_w( X_{(j)})) + \gamma \geq f^{c'}_w( X^{}_{(j)})] \\& = \err_P(f_w) - \err^\gamma_S(f_w) \leq \frac{2\sqrt{2}}{\gamma m} ({\Gamma(w) } +{\frac{1}{D^2N^2}}) D^{2} \cdot \\& (1 + \sqrt{2\log (4L \mathcal{C} D^4 N)}) \sqrt{\max_{t, k} \sum_{j=1}^m (X_{(j)_{tk}})^2} \frac{K}{(K-1)^2} \\&+ 3\sqrt{\frac{\log(2/\delta)+2\log({D^2N^2\Gamma(w) }+2)}{2m}},
\end{aligned}
\end{equation*}
}
where the maximum is taken over \(t \in [D]\),  \(k \in [L]\). 
\end{theorem} 
See Appendix~\ref{theorem:genproof} for the complete details and proof. A uniform convergence bound for a class $\mathcal{F}$ is an upper bound on the generalization gap that uniformly holds for all $f \in \mathcal{F}$. The more direct Rademacher complexity bound presented in  Lemma~\ref{lem:loss_ramp} 
is an example of a uniform convergence bound. However, these bounds become ineffective when a function $f \in \mathcal{F}$ exists that can perfectly fit any labeling of the dataset. In such situations, uniform convergence bounds fail to provide meaningful insights and are considered vacuous.

In contrast, the bound derived in Thm.~\ref{theorem:genbound} is also based on Rademacher complexity, but it is not a uniform convergence bound and hence it is not inherently vacuous. In the proof, we partition the class $\mathcal{F}$ into subsets $\mathcal{F}_{\rho}$, where the partitioning criterion $\rho$ is based on the norm of the S6 matrices, and apply our Rademacher bound (see Appendix \ref{theorem:rad}) 
within each subset. We then apply a union bound to combine these results, obtaining a bound 
individualized for each $f_w \in \mathcal{F}$. 


To understand our bound, we analyze its terms. 
First note that from the standard assumption that the data is normalized, we have,
\vspace{-8pt}
{\small
\begin{equation}
\sqrt{ \max_{t,k} \sum_{j=1}^m (X_{(j)_{tk}})^2} \leq \sqrt{m}   
\vspace{-2pt}
\end{equation}
}
%
%
Regarding the term \(\sqrt{\log(\mathcal{C}DNL)}\), even when choosing exceptionally large values for \(D\) or \(L\), such as \(2^{100}\), the impact on the bound remains minimal. The second term in the bound (see Thm~\ref{theorem:genbound}) is typically smaller than the first term, as it scales with \(\sqrt{\log(DN\Gamma(w))/m}\). Therefore, we conclude that our generalization bound scales with \(\mathcal{O}\left(\frac{1}{\sqrt{m}}D^{2} \Gamma(w) \cdot  \frac{K}{(K-1)^2}\right)\). This implies that the bound is largely unaffected by the sequence length \(L\), making it applicable to various sequence lengths without being significantly impacted by them. Note that since \( K < 1 \), when \( K \) is small, the term \(\frac{1}{1-K}\) approaches 1, further reducing its impact on the bound. This implies that for very small \( K \), the generalization bound becomes even tighter.
%

%% file: appendix_experiments.tex

\section{Experimental Setup \label{sec:hyperParams}} 
All training experiments were performed
on public datasets using a single A100 40GB GPU for a
maximum of two days. All experiments were conducted using PyTorch, and results are averaged over three seeds. All hyperparameters are detailed in Tab.~\ref{tab:NLPhyperpams} and Tab.~\ref{tab:Vsionhyperpams}.

\begin{table}[h]
\centering
\small
\begin{tabular}{l c}
\toprule
\textbf{Parameter} & \textbf{Value} \\
\midrule
Model-width & 192 \\
Number of layers & 24 \\
Number of patches & 196 \\
Batch-size & 512 \\
Optimizer & AdamW \\
Momentum & \( \beta_1, \beta_2 = 0.9, 0.999 \) \\
Base learning rate & $5e-4$ \\
Weight decay & 0.1 \\
Dropout & 0 \\
Training epochs & 300 \\
Learning rate schedule & cosine decay \\
Warmup epochs & 5 \\
Warmup schedule & linear \\ 
Degree of Taylor approx. (Eq.~\ref{eq:simplifiedModel}) & 3 \\
\bottomrule
\end{tabular}
\caption{Hyperparameters for image-classification via Vision Mamba variants} 
\label{tab:Vsionhyperpams}
\end{table}

\begin{table}[h]
\centering
\small
\begin{tabular}{l c}
\toprule
\textbf{Parameter} & \textbf{Value} \\
\midrule
Model-width & 386 \\
Number of layers & 12 \\
Context-length (training) & 1024 \\
Batch-size & 32 \\
Optimizer & AdamW \\
Momentum & \( \beta_1, \beta_2 = 0.9, 0.999 \) \\
Base learning rate & $1.5e-3$ \\
Weight decay & 0.01 \\
Dropout & 0 \\
Training epochs & 20 \\
Learning rate schedule & cosine decay  \\
Warmup epochs & 1  \\
Warmup schedule & linear  \\ 
Degree of Taylor approx. (Eq.~\ref{eq:simplifiedModel}) &  3\\
\bottomrule
\end{tabular}
\caption{Hyperparameters for language modeling via Mamba-based LMs} 
\label{tab:NLPhyperpams}
\end{table}

%% file: appendix_proofs.tex
\section{Proofs}\label{app:proofs}
This section details our proofs.
\subsection{Expressivity}
\setcounter{section}{2}
\setcounter{theorem}{0}
\setcounter{lemma}{0}

\setcounter{section}{1}
\begin{lemma}\label{lemma:dir2appendix}
There exists a function \( f : \mathbb{R}^L \rightarrow \mathbb{R} \) that can be realized by one channel of S6 such that a single attention head would require at least  \( O(\log(L)) \) layers to express this function.
\end{lemma}

\begin{proof}[Proof of Lemma~\ref{lemma:dir2appendix}]\quad
The proof begins by characterizing the functions that an S6 layer can implement as multivariate polynomials with a maximal degree that scales linearly with the sequence length $L$. It then demonstrates that this property does not hold for self-attention layers.

\noindent\textbf{Single S6 Layer as Multivariate Polynomials\quad} 
Let $f$ be a function implemented by a S6 layer with the parameters $A, S_{\Delta}, S_{B}$ and $S_{C}$.

\smallskip
Recall that we deal with the following polynomial variant of S6, which is defined as follows:
\begin{equation}\label{eq:simplifiedModelAPPEND23}
    C_i = ( {S_C}X_i )^T, \quad \bar{B}_i = S_B X_i, \quad \bar{A}_i = P_A ( S_\Delta X_i )
\end{equation}
\begin{equation}\label{eq:simplifiedModelAPPEND234}
     H_{i} =  \bar{A}_t H_{i-1} + \bar{B}_i X_i, \quad Y_k = C_i X_i
\end{equation}

Since we are interested in identifying the minimal polynomial degree required to characterize S6 models, we can assume that $P_A$ is linear, namely $\bar{A}_i = S_\Delta X_i + A $. By plugging the zero matrix into $A$ we get:

Now we can write Eq.~\ref{eq:simplifiedModelAPPEND23} by:

\begin{equation}\label{eq:simplifiedModelAPPEND25}
     H_{i} =  S_\Delta X_i H_{i-1} + \bar{B}_i X_i, \quad Y_i = C_i X_i
\end{equation}

Now, assuming $S_{\Delta}, S_{B}$ and $S_{C}$ are sparse matrices such that they have zeros at all elements except a single column, namely, the time-variant matrices are controlled only by the first channel of the sequence $X$. Hence, this channel can be defined by:
\begin{equation}\label{eq:simplifiedModelAPPEND6}
     h_{i} =  S_\Delta x_i h_{i-1} + \bar{B}_i x_i, \quad y_i = C_i x_i
\end{equation}

This exact recurrent rule is discussed in Eq.~\ref{eq:simplifiedModel} (Single S6 layer as multivariate polynomials), and it can be represented as a multivariate polynomial with a maximum degree of $L+2$. Hence, we can deduce that one element in the output of an S6 layer has a minimal maximum degree of $L + 2$ when processing sequences of length $L$.

\smallskip
\noindent\textbf{Single Self-Attention Layer as Multivariate Polynomials\quad} 
Given an input matrix $X$, the self-attention mechanism without softmax operates as follows. The input is projected into query $Q$, key $K$, and value $V$ matrices using parameter matrices $W_Q$, $W_K$, and $W_V$, respectively:
\begin{align*}
Q = X W_Q, \quad K = X W_K, \quad V = X W_V.
\end{align*}
The attention scores $A$ are then computed as:
\begin{align*}
A = Q K^T = (X W_Q)(X W_K)^T.
\end{align*}
The output matrix $Y$ is calculated by:
\begin{align*}
Y = A V = (X W_Q W_K^T X^T)(X W_V).
\end{align*}
This formulation leads to the conclusion that each element in $Y$ can be expressed as a multivariate polynomial with a maximum degree of 3 in the elements of $X$, where the polynomial arises from trilinear interactions. Additionally, it is important to note that in the case of causal attention, the mechanism is more constrained, and the maximum degree does not exceed 3.

Now, as it is clear that each single attention layer can be represented by a multivariate polynomial with a maximum degree of 3, we can generalize our characterization for $N$-stacked self-attention layers. Recall that the composition of multivariate polynomials is also a multivariate polynomial. Moreover, the maximum degree of the resulting polynomial is the product of the maximum degrees of the composed polynomials. Hence, we can argue that each element in the output of $N$-stacked self-attention layers can be represented by multivariate polynomials with a maximum degree $p \leq 3^N$. Therefore, it is clear that to represent a polynomial with a maximum degree of $L+2$ by $N$-stacked self-attention layers, at least $N \in O(\log L)$ layers are required.
\end{proof}

\begin{lemma}\label{lemma:dir1Appendix}
Any function that can be expressed by a single attention head can also be expressed by a single channel of S6.
\end{lemma}

\begin{proof}[Proof of Lemma~\ref{lemma:dir1Appendix}]\quad 
Let $f$ be a function implemented via a single attention head, which has the parameters $W_K$, $W_V$, $W_Q$.

Recall that we deal with the following polynomial variant of S6, which is defined as follows:
\begin{equation}\label{eq:simplifiedModelAPPEND}
\small
    C_i = ( {S_C}X_i )^T, \quad \bar{B}_i = S_B X_i, \quad \bar{A}_i = P_A ( S_\Delta X_i )
\end{equation}
\begin{equation}
\small
     H_{i} =  \bar{A}_t H_{i-1} + \bar{B}_i X_i, \quad Y_k = C_i X_i
\end{equation}

For simplicity, assume that we are dealing with a 1-degree polynomial $P_A$ such that $  \bar{A}_i = S_\Delta X_i + A $. 

By substituting the zero matrix for $S_\Delta$ and $A=1$, and plugging $S_C = W_Q$ and $S_B = W_K$, we get:
\begin{equation}
\small
    C_i = (W_Q X_i)^T, \quad \bar{B}_i = W_K X_i, \quad \bar{A}_i = \bar{A}_{i-1}
\end{equation}

By simply unrolling this equation:

\begin{equation} \label{eq:unrolling3a}
\small
 H_{i} = \sum_{j=1}^i \big{(} \Pi_{k=j+1}^i \bar{A}_k \big{)} W_K X_{j} X_j = \sum_{j=1}^i W_K X_{j} X_j, \quad %
 \end{equation}
\begin{equation} \label{eq:unrolling4a}
\small
 Y_{i} =  {X_i}^T {W_Q}^T \sum_{j=1}^i W_K X_{i} X_j =  \sum_{j=1}^i {X_i}^T {W_Q}^T W_K X_{j} X_j
\end{equation}

By converting Eq.~\ref{eq:unrolling4a} into a matrix form:

\begin{equation}\label{eq:MAMbaASmatmul1}
\small
Y = \tilde{\alpha} X
\end{equation} 
where $\tilde{\alpha}$ is defined by:
\begin{equation}\label{eq:MAMbaASmatmul}\nonumber
\tiny
\hspace{-11pt}
\begin{bmatrix}
    {X_1}^T {W_Q}^T  W_K X_1 & 0 & \hspace{-2pt}\cdots & 0 \\
    {X_2}^T {W_Q}^T  W_K X_1 & {X_2}^T {W_Q}^T  W_K X_2 & \hspace{-2pt}\cdots & 0 \\
    \vdots & \vdots &\hspace{-2pt} \ddots & 0 \\
    {X_L}^T {W_Q}^T   W_K X_{1} \quad & {X_L}^T {W_Q}^T  W_K X_{2} \quad & \hspace{-2pt}\cdots \quad & {X_L}^T {W_Q}^T  W_K X_L
\end{bmatrix}
\end{equation}

Which is the exact formulation of causal self-attention (without softmax) using attention matrices denoted by $\tilde{\alpha}$. Furthermore, to incorporate the value matrix $W_V$ an additional linear layer should be applied after step S6. These layers are indeed present in the Mamba block.
\end{proof}
\setcounter{theorem}{1}
\begin{theorem}\label{theorem:AnyPolywithMambasAppendix}
Given an input scalar sequence $x \in \mathbb{R}^L $, a model with four stacked Mamba layers, a sufficiently large number of channels, learnable PE, and a linear encoder at the first layer can express any multivariate polynomial of $x$.
\end{theorem}

\begin{proof}[Proof of Theorem~\ref{theorem:AnyPolywithMambasAppendix}]\quad
We start with the following definition of our model. We begin by describing the hidden Mamba layers, followed by the input and output layers:

\smallskip
\noindent\textbf{Model Definition\quad} We define the model with d channels, ignoring its activations. We denoted the \( i \)-th Mamba block by (see Eq.~\ref{eq:mamba1}):

\begin{equation}
    \small
    U^{i+1} = {\text{Linear}_1}^i {(\text{S6}}^i({\text{Conv1D}}^i({\text{Linear}_2}^i(U^i)) \otimes {\text{Linear}_3}^i(U^i) )
\end{equation}

where sequences and sub-layers associated with the \( i \)-th layer are denoted by the super-index \( i \). Thus, the entire computation of single layer can be described by

\begin{equation}
    \small
    U^{i+1} = \text{Mamba}^i (U_i)
\end{equation}

The output linear layer projects the last token (which include d channels) into a single output, parameterized by a matrix \( W_{\text{out}} \in \mathbb{R}^{d \times 1} \). The input layer includes the learnable positional encoding, represented by a matrix \( PE \in \mathbb{R}^{L \times d} \), and an encoding layer parameterized by \( \text{Encoding}(x) = W_{\text{in}}x + b + PE \), where \( W_{\text{in}}, b \in \mathbb{R}^{1 \times d} \).

\smallskip
\noindent\textbf{Proof by Construction\quad}
Let \( P(x) \) be a multivariate polynomial with coefficients \( c_1, \cdots, c_T \) and variables \( x = (x_1, x_2, \ldots, x_n) \). Specifically, \( P(x) \) can be expressed as:

\begin{equation}
P(x) = \sum_{i=1}^{T} c_i \cdot P_i(x), \quad \forall i: P_i (x) = \Pi_{j=1}^L \alpha_{i,j} {x_j}^{p_{i,j}}
\end{equation}

We assign the values of \( c_1, \cdots, c_T \) to \( W_{\text{out}} \) such that:
\begin{equation}
W_{\text{out}}[i,1] = 
\begin{cases} 
       c_i & \text{if } i \leq T, \\
       0 & \text{otherwise}.
\end{cases}
\end{equation}

It remains to show that for any \( i \), \( P_i(x) \) can be expressed by the last (4th) Mamba block. From Lemma~\ref{lemma:3layerMambaExpresivityAppendix}, it is evident that for any \( j \), the univariate polynomial \( s_j = \alpha_{i,j} x_j^{p_{i,j}} \) can be represented by the 3rd Mamba block. Thus, in the first linear layer of the 3rd block (\( \text{Linear}_1^3 \)), these univariate polynomials can be merged, as follows:

Define the following two sequences with Lemma 3 and $ \text{Linear}_1^3$:
\[
S' = (0, \cdots, 0, 1, s_1, s_2, \cdots, s_L, 1),
\]
\[
S'' = (0, \cdots, 0, 0, s_1, s_2, \cdots, s_L, 1),
\]
which are injected into the last SSM (\( \text{S6}^4 \)).

When applying the same SSM with system matrices equal to 1 (\( A_i = B_i = C_i = 1 \)), and subtracting them, placing the result in the i-th channel (which can be easily implemented by the last linear layer in the final Mamba block), yield:

\begin{align}
\forall i: U^{4}_i = \prod_{j=1}^L s_j = P_i(x), \quad \rightarrow W_{out}U^4 = P(x)
\end{align} 
as required.
\end{proof}

\begin{lemma}\label{lemma:3layerMambaExpresivityAppendix}
Given an input scalar sequence $x \in \mathbb{R}^L$, for any $j$, a model $M$ with 3-stacked Mamba layers, a sufficiently large number of channels, learnable PE, and a linear encoder in the first layer can express any monomial of a univariate polynomial in $x_j$. Specifically, for any constants $c \in \mathbb{R}$ and $P \in \mathbb{R}$, there exists a configuration of $M$ such that the output of the $k$-th channel $M(x)_k = c \cdot x_j^P$ for $k > P + j$.
\end{lemma}

\begin{proof}[Proof of Lemma~\ref{lemma:3layerMambaExpresivityAppendix}]\quad
\smallskip
Given an input sequence \( x = (x_1, x_2, \dots, x_L) \) with \( d \) channels, we need to show that a model \( M \) composed of three stacked Mamba layers can express any univariate polynomial \( x_j^P \) for a particular \( j \), where \( P \) are constant.

For simplicity, we assume that $P <<L$. However, this assumption can be addressed by extending the sequence length using the ZeroPad component.

\noindent \textbf{Step 1: Position Selection (First Mamba Block)} 

The first block's role is to isolate the desired position \( j \) in the input sequence $(x_j)$ at channel $s$. We achieve this as follows:

\begin{itemize}
    \item \textbf{Positional Encoding and Gating:} We configure the learnable positional encoding at channel $s$ (\( PE_{*,s} \)) such that the \( j \)-th position is highlighted, and the others positions are masked. Specifically, set the PE vector for the \( j \)-th position to act as an indicator function:
    \[
    PE[i,s] = \begin{cases} 
       1 & \text{if } i = j, \\
       0 & \text{otherwise}.
    \end{cases}
    \]\

    Please note that $s> \frac{d}{2}$, and the first half of the channels in the PE are set to zeros to ensure a clear separation between the input values $X$ and the positional encoding. Additionally, we set $b=0$, and configure the first half of the entries in $W_{in}$ to 1, while setting the remaining entries to zero. This configuration effectively duplicates the input sequence, allowing for several manipulations without affecting the original signal.
    
    \item \textbf{Linear Layers Configuration:} The linear layers \( \text{Linear}_2^1 \) and \( \text{Linear}_3^1 \) are configured to be identity mappings, i.e., they do not alter the input sequence.
    
    \item \textbf{S6 Module Configuration:} Set the S6 parameters \( \bar{A} = 0 \) and \( \bar{B} = \bar{C} = 1 \), effectively making it an identity operation. This setup ensures that the output of this block isolates \( x_j \) while setting all other positions to zero:
    \[
    U^2 = (0, \dots, 0, x_j, 0, \dots, 0)
    \]
\end{itemize}

\noindent \textbf{Step 2: Duplication of \( x_j \) (Second Mamba Block)}

The second block is responsible for duplicating the selected element \( x_j \) to match the desired power \( P \). Here's how:

\begin{itemize}
    \item \textbf{Identity Mapping:} We set \( \text{Linear}_2^2 \) and \( \text{Linear}_3^2 \) to identity mappings. Additionally, the gate branch on the channels operating on the isolated \( x_j \) is set to 1 at positions smaller than $h$ and 0 for the rest, also functioning as an identity mapping at positions $i \leq h$, and masking positions for indexes $i>j$.

    \item \textbf{S6 :} The S6 module is configured to allow duplication. Specifically, by setting \( \bar{A} = 1 \) and \( \bar{B}, \bar{C} = 1 \), the S6 can output a sequence with multiple copies of \( x_j \):
    \[
    U^3 = (0, \dots, 0, x_j, x_j, \dots, x_j, 0, \dots, 0)
    \]
    Here, \( x_j \) appears \( P-2 \) times, and the last $x_j$ positioned at index $i=h$.
\end{itemize}

\noindent \textbf{Step 3: Aggregate Multiplications to Powers (Third Mamba Block)}

The third Mamba block is designed to aggregate the duplicated elements \( x_j \) into the form \( x_j^P \), utilizing the multiplicative capabilities of the S6 module and the subtraction mechanism in the final linear layer.

\begin{itemize}
    \item \textbf{S6 Module Configuration:} In this block, the S6 module is configured to perform the necessary multiplications that aggregate the duplicated values of \( x_j \). This is achieved by setting the system input and output matrices to 1. Hence, the output of the SSM when applied on $U_2$ at position $h$ will result at:
    
    \begin{equation}
    \small
    \sum_{j=1}^{P-2} c_j {x_j}^{j+2}
    \nonumber
    \end{equation}

Thus, we construct an additional sequence, similar to $U_2$, denoted by $U_2'$, by introducing an additional zero at the initial occurrence of $x_j$. We then subtract the outputs from these two identical SSM channels at the final linear layer of the block. This subtraction yields a telescoping series:
    \begin{equation}
    \text{SSM}(U_2) - \text{SSM}(U_2') = 
    \end{equation}
     
    \begin{equation}
    \small
    \sum_{j=1}^L c_j \Big{(} \Pi_{k=j+1}^{t-1} x_k \Big{)} {x_t}^2 {x_j}^2 - \sum_{j=2}^L c_j \Big{(} \Pi_{k=j+1}^{t-1} x_k \Big{)} {x_t}^2 {x_j}^2 
%
    =\sum_{j=1}^{P-2} c_j {x_j}^{j+2} - \sum_{j=2}^{P-2} c_j {x_j}^{j+2} = {x_j}^{P}
\end{equation}
yields
   \[
    U^4 = \left(0, \dots, 0, x_j^P, 0, \dots, 0\right)
    \]
\end{itemize}

This construction shows that a model with three stacked Mamba layers and sufficient channels can indeed express any univariate polynomial \( x_j^P \), thereby proving Lemma~\ref{lemma:3layerMambaExpresivity}.

\end{proof}

%% file: appendix_generalization.tex
\setcounter{section}{2}
\subsection{Generalization}

Let $P$ be a distribution over $\mathbb{R}^{D \times L} \times [C]$. Let $S = \{( X^{}_{(j)},y_{(j)})\}^{m}_{j=1}$ be a dataset of i.i.d. samples selected from $P$. Our generalization bound is based on a uniform-convergence generalization bound provided in ~\citep{galanti2024norm}. The following lemma bounds the gap between the test error and the empirical margin error, represented as $\err^{\gamma}_S(f_w)= 
\fr{1}{m}\sum^{m}_{j=1}\bI[\max_{c \neq c'}(f^c_w( X^{}_{(j)})) + \gamma \geq f^{c'}_w( X^{}_{(j)})]$. 

\begin{lemma}\label{lem:loss_ramp}
Let $P$ be a distribution over $\mathbb{R}^{\mathcal{D}} \times [C]$ and $\mathcal{F} \subset \{f':\mathcal{X} \to \R^C \}$. Let $S = \{(X_j,y_j)\}^{m}_{j=1}$ be a dataset of i.i.d. samples selected from $P$ and $X=\{X_j\}^{m}_{j=1}$. Then, with probability at least $1-\delta$ over the selection of $S$, for any $f_w \in \mathcal{F}$, we have
\begin{equation}\label{eq:Radbound}
\err_P(f_w) - \err^{\gamma}_S(f_w) \leq \frac{2\sqrt{2}}{\gamma} \cdot \mathcal{R}_{X}(\mathcal{F}) + 3\sqrt{\frac{\log(2/\delta)}{2m}}.
\end{equation} 
\end{lemma}
\begin{lemma}\label{lem:peeling}
Let $\sigma_j$ be a $l_j$-Lipschitz, positive-homogeneous function and $l=\max_j l_j$. Let $\xi_i \sim U[\{\pm 1\}]$. Then for any class of vector-valued functions $\mathcal{F} \subset \{f \mid ~f:\mathbb{R}^d \to \mathbb{R}\}$ and any convex and monotonically
increasing function $g : \mathbb{R} \to [0,\infty)$, 
\begin{small}
\begin{equation*}
\begin{aligned}
\E_{\xi} \sup_{\substack{f\in \mathcal{F}}} g\left( \left| \sum^{m}_{j=1} \xi_j \cdot \sigma_j(f(x_j))\right| \right)
~\leq~ 2\E_{\xi} \sup_{f\in \mathcal{F}} g\left(l\left|\sum^{m}_{i=j} \xi_j  \cdot f(x_j)\right| \right).
\end{aligned}
\end{equation*}
\end{small}
\end{lemma}
\begin{proof}
We notice that since $g(|z|) \leq g(z)+g(-z)$,
\begin{small}
\begin{equation*}
\begin{aligned}
&\E_{\xi} \sup_{\substack{f\in \mathcal{F}}} g\left( \left| \sum^{m}_{j=1} \xi_j \cdot \sigma_j(f(x_j))\right| \right)\\
&\leq \E_{\xi} \sup_{\substack{f\in \mathcal{F}}} g\left( \sum^{m}_{j=1} \xi_j \cdot \sigma_j(f(x_j)) \right)\\
&+\E_{\xi} \sup_{\substack{f\in \mathcal{F}}} g\left( - \sum^{m}_{j=1} \xi_j \cdot \sigma_j(f(x_j)) \right) \\
&= 2\E_{\xi} \sup_{\substack{f\in \mathcal{F}}} g\left( \sum^{m}_{j=1} \xi_j \cdot \sigma_j(f(x_j)) \right)\\
\end{aligned}
\end{equation*}
\end{small}
where the last equality follows from the symmetry in the distribution of the $\xi_i$ random variables. By Equation 4.20 in~\cite{Ledoux1991ProbabilityIB}  we have the following:
\begin{small}
\begin{equation*}
\begin{aligned}
&\E_{\xi} \sup_{\substack{f\in \mathcal{F}}} g\left( \sum^{m}_{j=1} \xi_j \cdot \sigma_j(f(x_j)) \right)\\
&=\E_{\xi} \sup_{\substack{f\in \mathcal{F}}} g\left(l \sum^{m}_{j=1} \xi_j \cdot \frac{1}{l}\sigma_j(f(x_j)) \right)\\
&\leq\E_{\xi} \sup_{\substack{f\in \mathcal{F}}} g\left(l \sum^{m}_{j=1} \xi_j \cdot f(x_j) \right)\\
&\leq\E_{\xi} \sup_{\substack{f\in \mathcal{F}}} g\left(l \left| \sum^{m}_{j=1} \xi_j \cdot f(x_j) \right| \right)\\
\end{aligned}
\end{equation*}
\end{small}
where the last equality follows from $g$ being monotonically increasing. 
\end{proof}
The model is denoted by:
\begin{equation}\label{eq:modelAppendix}
\begin{aligned} 
      &B_i = S_B x_i, \quad C_i = S_C x_i, \quad \Delta_i = \sigma(S_{\Delta} x_i), \quad  \\&\bar{A}_i =  \exp(\Delta_i A), \quad \bar{B}_i = B_i 
\end{aligned}
\end{equation}
Where \( \sigma \) is a 1-Lipschitz activation function. 
We consider a classifier $f : \mathbb{R}^{D \times L} \rightarrow \mathbb{R}^{\mathcal{C}}$ defined as follows. We have parameters $(A^{}, S_B^{}, S_C^{}, S_{\Delta}^{})$ associated with the layer. The norms are defined as follows:

\begin{equation}
\begin{aligned} 
    \rho_A^{}(w) &= \|A^{}\|_{\max} \\
    \rho_B^{}(w) &= \|S_B^{}\|_{2,\infty} \\
    \rho_C^{}(w) &= \|S_C^{}\|_{F} \\
    \rho_{\Delta}^{}(w) &= \|S_{\Delta}^{}\|_{2}
\end{aligned}
\end{equation}

We denote the product of these norms as:
\begin{equation}
\Gamma^{}(w) = \rho_A^{}(w) \cdot \rho_B^{}(w) \cdot \rho_C^{}(w) \cdot \rho_{\Delta}^{}(w)
\end{equation}

Given these definitions, the classifier $f$ for a specific class $c \in [\mathcal{C}]$ is computed as:

\begin{small}
\begin{equation*}
\begin{aligned}
&f^c(X_{*1},...,X_{*L}) =  
\sum_{d=1}^D W_{c,d} \left(S_C^{} X^{}_{*L}\right)^T \sum_{i=1}^L \left(\prod_{k=i+1}^L \bar{A}^{}_{dk} \right) S_B^{} X^{}_{*i} X^{}_{di}
\end{aligned}
\end{equation*}
\end{small}

Here, $W \in \mathbb{R}^{\mathcal{C} \times D}$ represents a linear projection from the output to the number of classes, and $\mathcal{C}$ is the number of classes. 

We denote the parameters of the classifier by 
\[
w = (A^{}, S_B^{}, S_C^{}, S_{\Delta}^{}, W) 
\] and the function induced by a specific instance of $w$ is denoted by $f_w$. The class of functions taking on different parameter instances $w$ is denoted by $\mathcal{F}$.
Denote 
\[
\rho = \{\rho_W, \rho_A, \rho_B, \rho_C, \rho_{\Delta}\}
\]
and let:
\begin{equation}
\begin{aligned}
   \mathcal{F}_{\rho} 
    &=\{f_w \in \mathcal{F} : \Gamma(w) \leq  \rho_A \rho_B \rho_C \rho_{\Delta} \rho_W=:\Gamma \}
\end{aligned}
\end{equation}

The following theorem provides a bound on the Rademacher complexity of the class $\mathcal{F}_{\rho}$.
\setcounter{section}{2}
\setcounter{theorem}{3}
\begin{theorem}\label{theorem:rad}
Let $\rho = \{\rho_W, \rho_A, \rho_B, \rho_C, \rho_{\Delta}\}$. Suppose we have \(m\) sample sequences \(X = \{ X_{(j)} \}_{j=1}^{m}\), where each \(X_{(j)} = (X_{(j)_{*1}}, \ldots, X_{(j)_{*L}}) \in \mathbb{R}^{D \times L}\). Assume that $\forall j \in [m]: ||X_{(j)}||_{\max} \leq 1$. Additionally, suppose $\forall k \in [L], d \in [D]:||\bar{A}^{}_{dk}||_{\max} < K < 1$. Then,
\begin{small}
\begin{equation*}
\begin{aligned}
\mathcal{R}_X(\mathcal{F}_{\rho})
&\leq\frac{1}{m} D^{2} \Gamma (1 + \sqrt{2\log (2L \mathcal{C} D^4 N)}) \cdot  \sqrt{\max_{t, k} \sum_{j=1}^m (X_{(j)_{tk}})^2} \frac{K}{(K-1)^2}
\end{aligned}
\end{equation*}
\end{small}
where the maximum is taken over \(t \in [D]\),  \(k \in [L]\). 
\end{theorem}

\begin{proof}
For aesthetic purposes, we define \( B \) as \( S_B \) and \( C \) as \( S_C \).
The Rademacher complexity of $\mathcal{F}_{\rho}$ is given by:
\begin{small}
\begin{equation*}
\begin{aligned}
&m\mathcal{R(\mathcal{F}_{\rho})} = \E_{\xi}\left[\sup_{w} \sum_{j=1}^m \sum_{c=1}^{\mathcal{C}} \xi_{jc} f^c_w(X_{(j)}) \right] \\
&= \E_{\xi}\left[\sup_{w} \sum_{j=1}^m \sum_{c=1}^{\mathcal{C}} \xi_{jc} \sum_{d=1}^D W_{c,d} 
\left(S_C^{} X^{}_{{(j)}_{*L}}\right)^T \sum_{i=1}^L \right. \left.\left(\prod_{k=i+1}^L \exp\left(\sigma(S^{}_{\Delta} X^{}_{{(j)}_{*k}}) 
\cdot A^{}_{d*}\right)\right) S_B^{} X^{}_{{(j)}_{*i}} X^{}_{{(j)}_{di}} \right] \\
&= \E_{\xi}\left[\sup_{w} \sum_{j=1}^m \sum_{c=1}^{\mathcal{C}} \xi_{jc} \sum_{d=1}^D W_{c,d} 
\left(C^{} X^{}_{{(j)}_{*L}}\right)^T \sum_{i=1}^L \right. \left.\left(\prod_{k=i+1}^L \exp\left(\sigma(S^{}_{\Delta} X^{}_{{(j)}_{*k}}) 
\cdot A^{}_{d*}\right)\right) B^{} X^{}_{{(j)}_{*i}} X^{}_{{(j)}_{di}} \right] 
\end{aligned}
\end{equation*}
\end{small}

Here, {$A^{}_{d*}$ represents the $d$-th row of $A^{}$}, which has a size of $N$. We think of $A^{}_{d*}$ as a diagonal matrix of size $N\times N$, where its diagonal elements are the values in the $d$th row of $A$. Thus, $A^{}_{d*} = \text{diag}(a^{}_{d1},...,a^{}_{dN})$. Hence,
\begin{small}
\begin{equation*}
\begin{aligned}
& \E_{\xi}\left[\sup_{w} \sum_{j=1}^m \sum_{c=1}^{\mathcal{C}} \xi_{jc} \sum_{d=1}^D W_{c,d} 
\left(C^{} X^{}_{{(j)}_{*L}}\right)^T \sum_{i=1}^L \right. 
 \left. \left(\prod_{k=i+1}^L \exp\left(\sigma(S^{}_{\Delta} X^{}_{{(j)}_{*k}}) \cdot 
A^{}_{d*}\right) \right) B^{} X^{}_{{(j)}_{*i}} X^{}_{{(j)}_{di}} \right] \\
&=~ \E_{\xi}\left[\sup_{w} \sum_{j=1}^m \sum_{c=1}^{\mathcal{C}} \xi_{jc} \sum_{d=1}^D W_{c,d} 
\sum_{l=1}^N \left(C^{} X^{}_{{(j)}_{*L}}\right)_l^T \sum_{i=1}^L \right. \left. \left(\prod_{k=i+1}^L \exp\left(\sigma(S^{}_{\Delta} X^{}_{{(j)}_{*k}}) \cdot 
a^{}_{dl}\right) \right) B_{l*}^{} X^{}_{{(j)}_{*i}} X^{}_{{(j)}_{di}} \right] \\
&=~ \E_{\xi}\left[\sup_{w} \sum_{j=1}^m \sum_{c=1}^{\mathcal{C}} \xi_{jc} \sum_{d=1}^D W_{c,d} 
\sum_{l=1}^N \left(C_{l*}^{} X^{}_{{(j)}_{*L}}\right) \sum_{i=1}^L \right. \left. \left(\prod_{k=i+1}^L \exp\left(\sigma(S^{}_{\Delta} X^{}_{{(j)}_{*k}}) \cdot 
a^{}_{dl}\right) \right) B_{l*}^{} X^{}_{{(j)}_{*i}} X^{}_{{(j)}_{di}} \right] \\
&=~ \E_{\xi}\left[\sup_{w} \sum_{c=1}^{\mathcal{C}} \sum_{d=1}^D W_{c,d} \sum_{j=1}^m \xi_{jc} 
\sum_{l=1}^N \left(C_{l*}^{} X^{}_{{(j)}_{*L}}\right) \sum_{i=1}^L \right.  \left. \left(\prod_{k=i+1}^L \exp\left(\sigma(S^{}_{\Delta} X^{}_{{(j)}_{*k}}) \cdot 
a^{}_{dl}\right) \right) B_{l*}^{} X^{}_{{(j)}_{*i}} X^{}_{{(j)}_{di}} \right]
\end{aligned}
\end{equation*}
\end{small}

This follows from expressing the model in a more explicit form. Next,
\begin{small}
\begin{equation*}
\begin{aligned}
&~ \E_{\xi}\left[\sup_{w} \sum_{c=1}^{\mathcal{C}}  \sum_{d=1}^D W_{c,d} \sum_{j=1}^m \xi_{jc} \sum_{l=1}^N \left(C_{l*}^{} X^{}_{{(j)}_{*L}}\right) \sum_{i=1}^L \right.  \left. \left(\prod_{k=i+1}^L \exp\left(\sigma(S^{}_{\Delta} X^{}_{{(j)}_{*k}}) \cdot a^{}_{dl}\right) \right) B_{l*}^{} X^{}_{{(j)}_{*i}} X^{}_{{(j)}_{di}} \right] \\
&=~ \sqrt{D} \rho_W \E_{\xi}\left[\sup_{w,c, d}   | \sum_{j=1}^m \xi_{jc} \sum_{l=1}^N \left(C_{l*}^{} X^{}_{{(j)}_{*L}}\right) \sum_{i=1}^L \right.\left. \left(\prod_{k=i+1}^L \exp\left(\sigma(S^{}_{\Delta} X^{}_{{(j)}_{*k}}) \cdot a^{}_{dl}\right) \right) B_{l*}^{} X^{}_{{(j)}_{*i}} X^{}_{{(j)}_{di}} | \right]
\end{aligned}
\end{equation*}
\end{small}
where the inequality follows from moving the norm of $W$ to $W_{c}$ for maximizing the inner term and applying the Cauchy-Schwartz inequality.
Next,
\begin{small}
\begin{equation*}
\begin{aligned}
&~ \sqrt{D} \rho_W \E_{\xi}\left[\sup_{w,c, d}   | \sum_{j=1}^m \xi_{jc} \sum_{l=1}^N \left(C_{l*}^{} X^{}_{{(j)}_{*L}}\right) \sum_{i=1}^L \right.  \left. \left(\prod_{k=i+1}^L \exp\left(\sigma(S^{}_{\Delta} X^{}_{{(j)}_{*k}}) \cdot a^{}_{dl}\right) \right) B_{l*}^{} X^{}_{{(j)}_{*i}} X^{}_{{(j)}_{di}} | \right] \\
&=~ \sqrt{D} \rho_W \E_{\xi}\left[\sup_{w,c, d}   | \sum_{j=1}^m \xi_{jc} \sum_{l=1}^N \left(\sum_{s=1}^D C_{ls}^{} X^{}_{{(j)}_{sL}}\right) \sum_{i=1}^L \right.  \left. \left(\prod_{k=i+1}^L \exp\left(\sigma(S^{}_{\Delta} X^{}_{{(j)}_{*k}}) \cdot a^{}_{dl}\right) \right) \left(\sum_{s'=1}^DB_{ls'}^{} X^{}_{{(j)}_{s'i}} \right) X^{}_{{(j)}_{di}} | \right] \\
&=~ \sqrt{D} \rho_W \E_{\xi}\left[\sup_{w,c, d}   | \sum_{l=1}^N \sum_{s=1}^D C_{ls}^{} \sum_{j=1}^m \xi_{jc}  X^{}_{{(j)}_{sL}} \sum_{i=1}^L \right. \left. \left(\prod_{k=i+1}^L \exp\left(\sigma(S^{}_{\Delta} X^{}_{{(j)}_{*k}}) \cdot a^{}_{dl}\right) \right) \left(\sum_{s'=1}^DB_{ls'}^{} X^{}_{{(j)}_{s'i}} \right) X^{}_{{(j)}_{di}} | \right] \\
&\leq~ D \rho^{}_W \rho^{}_C \E_{\xi}\left[\sup_{w,c, d, l, s} | \sum_{j=1}^m \xi_{jc}  X^{}_{{(j)}_{sL}} \sum_{i=1}^L \right.  \left. \left(\prod_{k=i+1}^L \exp\left(\sigma(S^{}_{\Delta} X^{}_{{(j)}_{*k}}) \cdot a^{}_{dl}\right) \right) \left(\sum_{s'=1}^DB_{ls'}^{} X^{}_{{(j)}_{s'i}} \right) X^{}_{{(j)}_{di}} |  \right] \\
&=~ D \rho^{}_W \rho^{}_C \E_{\xi}\left[\sup_{w,c, d, l, s} | \sum_{s'=1}^D B_{ls'}^{} \sum_{j=1}^m \xi_{jc}  X^{}_{{(j)}_{sL}} \sum_{i=1}^L \right.  \left. \left(\prod_{k=i+1}^L \exp\left(\sigma(S^{}_{\Delta} X^{}_{{(j)}_{*k}}) \cdot a^{}_{dl}\right) \right) X^{}_{{(j)}_{s'i}} X^{}_{{(j)}_{di}} |  \right]
\end{aligned}
\end{equation*}
\end{small}
where the first inequality follows from moving the norm of $C$ to $C_{l}$ for maximizing the inner term and applying the Cauchy-Schwartz inequality. 
\begin{small}
\begin{equation*}
\begin{aligned}
&=~ D \rho^{}_W \rho^{}_C \E_{\xi}\left[\sup_{w,c, d, l, s} | \sum_{s'=1}^D B_{ls'}^{} \sum_{j=1}^m \xi_{jc}  X^{}_{{(j)}_{sL}} \sum_{i=1}^L \right.  \left.\left(\prod_{k=i+1}^L \exp\left(\sigma(S^{}_{\Delta} X^{}_{{(j)}_{*k}}) \cdot a^{}_{dl}\right) \right) X^{}_{{(j)}_{s'i}} X^{}_{{(j)}_{di}} |  \right] \\
&\leq~ D^{1.5} \rho^{}_W \rho^{}_C \rho^{}_B \E_{\xi}\left[\sup_{w,c, d, l, s, s'} | \sum_{j=1}^m \xi_{jc}  X^{}_{{(j)}_{sL}} \sum_{i=1}^L \right. \left. \left(\prod_{k=i+1}^L \exp\left(\sigma(S^{}_{\Delta} X^{}_{{(j)}_{*k}}) \cdot a^{}_{dl}\right) \right) X^{}_{{(j)}_{s'i}} X^{}_{{(j)}_{di}} |  \right] \\
&=~ D^{1.5} \rho^{}_W \rho^{}_C \rho^{}_B \E_{\xi}\left[\sup_{w,c, d, l, s, s'} | \sum_{j=1}^m \xi_{jc}  X^{}_{{(j)}_{sL}} \sum_{i=1}^L  \right. \left.\exp\left( \sum_{k=i+1}^L \sigma(S^{}_{\Delta} X^{}_{{(j)}_{*k}}) \cdot a^{}_{dl}\right) X^{}_{{(j)}_{s'i}} X^{}_{{(j)}_{di}} |  \right] \\
&\leq~ D^{1.5} \rho^{}_W \rho^{}_C  \rho^{}_B \sum_{i=1}^L \E_{\xi}\left[\sup_{w,c, d, l, s, s'} | \sum_{j=1}^m \xi_{jc}  X^{}_{{(j)}_{sL}} \cdot  \right. \left. \exp\left( \sum_{k=i+1}^L \sigma(S^{}_{\Delta} X^{}_{{(j)}_{*k}}) \cdot a^{}_{dl}\right) X^{}_{{(j)}_{s'i}} X^{}_{{(j)}_{di}} |  \right]
\end{aligned}
\end{equation*}
\end{small} where the second to last inequality follows from $\|x\|_2 \leq \sqrt{n} \|x\|_{\infty}$ for any $x \in \mathbb{R}^{n}$. 
Jensen's inequality gives the following inequality:
\begin{small}
\begin{equation}\label{eq:eq1}
\begin{aligned}
&~ \E_{\xi}\left[\sup_{w,c, d, l, s, s'} | \sum_{j=1}^m \xi_{jc}  X^{}_{{(j)}_{sL}} \cdot \right. \left. \exp\left( \sum_{k=i+1}^L \sigma(S^{}_{\Delta} X^{}_{{(j)}_{*k}}) \cdot a^{}_{dl}\right) X^{}_{{(j)}_{s'i}} X^{}_{{(j)}_{di}} |  \right] \\
&\leq~  \frac{1}{\lambda_i} \log ( \E_{\xi}\left[\sup_{w,c, d, l, s, s'}  \exp (\lambda_i  | \sum_{j=1}^m \xi_{jc}  X^{}_{{(j)}_{sL}} \cdot \right. \left.  \exp\left( \sum_{k=i+1}^L \sigma(S^{}_{\Delta} X^{}_{{(j)}_{*k}}) \cdot a^{}_{dl}\right) X^{}_{{(j)}_{s'i}} X^{}_{{(j)}_{di}} |)  \right] ) \\
&\leq~  \frac{1}{\lambda_i} \log ( \sum_{c, d, l, s, s'} \E_{\xi}\left[\sup_{w} \exp(\lambda_i  | \sum_{j=1}^m \xi_{jc}  X^{}_{{(j)}_{sL}} \cdot  \right. \left. \exp\left( \sum_{k=i+1}^L \sigma(S^{}_{\Delta} X^{}_{{(j)}_{*k}}) \cdot a^{}_{dl}\right) X^{}_{{(j)}_{s'i}} X^{}_{{(j)}_{di}} |)  \right] ) \\
&\leq~  \frac{1}{\lambda_i} \log ( \mathcal{C} D^3 N  \max_{c, d, l, s, s'} \E_{\xi}\left[\sup_{w} \exp(\lambda_i  | \sum_{j=1}^m \xi_{jc}  X^{}_{{(j)}_{sL}} \cdot \right.  \left. \exp\left( \sum_{k=i+1}^L \sigma(S^{}_{\Delta} X^{}_{{(j)}_{*k}}) \cdot a^{}_{dl}\right) X^{}_{{(j)}_{s'i}} X^{}_{{(j)}_{di}} |)  \right] ) :=\Theta \\
\end{aligned}
\end{equation}
\end{small}
For fixed $\lambda_i > 0$. The second inequality follows from the fact that $\sup_x \sup_y f(x,y) = \sup_{x,y} f(x,y)$.
We observe that the inner expectation $\sup$ depends only on $w$.
Next we use Lemma~\ref{lem:peeling} with $\sigma_{ij}(z) = \exp(z)X^{}_{{(j)}_{sL}} X^{}_{{(j)}_{s'i}} X^{}_{{(j)}_{di}}$ on its domain and $g_i(X_{(j)})=\sum_{k=i+1}^L \sigma(S^{}_{\Delta} X^{}_{{(j)}_{*k}}) \cdot a^{}_{dl}$. The corresponding Lipschitz constants are $l_{ij} = \max_{z \in dom(\sigma_{ij})}(\exp(z)X_{(j)_{sL}}X_{(j)_{s'i}} X_{(j)_{di}})$ and $l_i = \max(l_{ij})$. Therefore:
\begin{small}
\begin{equation}
\begin{aligned}
&\E_{\xi}\left[\sup_{w} \exp(\lambda_i | \sum_{j=1}^m \xi_{jc}  X^{}_{{(j)}_{sL}} \cdot \right. \left. \exp\left( \sum_{k=i+1}^L \sigma(S^{}_{\Delta} X^{}_{{(j)}_{*k}}) \cdot a^{}_{dl}\right) X^{}_{{(j)}_{s'i}} X^{}_{{(j)}_{di}} |)  \right] \\
&=\E_{\xi}\left[ \sup_{w}  \exp \left( \lambda_i  \left| \sum_{j=1}^m \xi_{jc} (\sigma_{ij}(g_i(X^{(j)}))) \right| \right)\right] \\
&\leq 2\E_{\xi}\left[ \sup_{w}  \exp \left( \lambda_i l_i \left| \sum_{j=1}^m \xi_{jc}    \sum_{k=i+1}^L \sigma(S^{}_{\Delta} X^{}_{{(j)}_{*k}}) \cdot a^{}_{dl} \right| \right)\right] \\
&\leq 2\E_{\xi}\left[ \sup_{w}  \exp \left( \lambda_i \rho^{}_A  l_i  \left| \sum_{j=1}^m \xi_{jc}    \sum_{k=i+1}^L \sigma(S^{}_{\Delta} X^{}_{{(j)}_{*k}}) \right| \right)\right] \\
&\leq 2\E_{\xi}\left[ \sup_{w}  \exp \left(\lambda_i \rho^{}_A  (L-i) l_i \left| \sum_{j=1}^m \xi_{jc}   \sigma(S^{}_{\Delta} X^{}_{{(j)}_{*k}}) \right| \right)\right] \\
&\leq 4\E_{\xi}\left[ \sup_{w}  \exp \left(\lambda_i \rho^{}_A  (L-i) l_i \left| \sum_{j=1}^m \xi_{jc}   S^{}_{\Delta} X^{}_{{(j)}_{*k}} \right| \right)\right] \\
\end{aligned}
\end{equation}
\end{small}
This follows from applying Lemma~\ref{lem:peeling} with $\sigma$ which has a Lipschitz constant of 1, as assumed. Hence:
\begin{small}
\begin{equation} \label{eq:lip}
\begin{aligned}
& 4\E_{\xi}\left[ \sup_{w}  \exp \left( \lambda_i \rho^{}_A  (L-i) l_i \left| \sum_{j=1}^m \xi_{jc} S^{}_{\Delta} X^{}_{{(j)}_{*k}} \right| \right)\right] \\
&\leq 4\E_{\xi}\left[ \sup_{k}  \exp \left( \lambda_i \rho^{}_A \rho^{}_{\Delta} (L-i) l_i || \sum_{j=1}^m \xi_{jc} X^{}_{{(j)}_{*k}} || \right)\right] \\
&\leq 4\E_{\xi}\left[ \sup_{t,k}  \exp \left( \lambda_i \sqrt{D} \rho^{}_A \rho^{}_{\Delta} (L-i) l_i \left| \sum_{j=1}^m \xi_{jc} X^{}_{{(j)}_{tk}} \right| \right)\right] \\
&\leq 4DL \max_{t, k} \E_{\xi}\left[ \exp \left(\lambda_i \sqrt{D} \rho^{}_A \rho^{}_{\Delta} (L-i) l_i  \left|  \sum_{j=1}^m \xi_{jc} X^{}_{{(j)}_{tk}}  \right| \right)\right] \\
\end{aligned}
\end{equation}
\end{small}

Denote:
\begin{small}
\begin{equation} \label{eq:lip2}
\begin{aligned}
M_i := \sqrt{D} \rho^{}_A \rho^{}_{\Delta} (L-i) l_i 
\end{aligned}
\end{equation}
\end{small}



As a next step, we would like to bound the above term using a function of the data that is not dependent on an expected value of noise labels $\xi$. For this purpose we apply a technique that was introduced in the proof of Theorem~1 in~\citep{golowich2018size}. We apply this process separately for each $i \in [L]$. Let $i \in [L]$:
We define a random variable $Z$:
\begin{small}
\begin{equation*}
\begin{aligned}
&Z = M_i  \left| \sum_{j=1}^m \xi_{jc} X^{}_{{(j)}_{tk}} \right| \\
&z_j = X^{}_{{(j)}_{tk}} \rightarrow Z = M_i |\sum_{j=1}^m \xi_{jc} z_j | \\
\end{aligned}
\end{equation*}
\end{small}
The random variable $Z$ depends on the random variables $\xi_{jc}$.
Then, we have: 
\begin{small}
\begin{equation*}
\begin{aligned}
&=\frac{1}{\lambda_i} \log \E_{\xi}\left[\exp( \lambda _iZ) \right] \\
&=\frac{1}{\lambda_i} \log \E_{\xi}\left[\exp( \lambda_i Z + \lambda_i\E(Z) - \lambda_i \E(Z)) \right] \\
&=\frac{1}{\lambda_i} \log \E_{\xi}\left[\exp( \lambda_i Z - \lambda_i \E(Z)) \right] +  \E_{\xi} (Z) \\
\end{aligned}
\end{equation*}
\end{small}
By Jensen’s inequality, we obtain a bound for $\E(|\sum_{j=1}^m \xi_{jc} z_j |)$:
\begin{small}
\begin{equation*}
\begin{aligned}
&\E_{\xi} (|\sum_{j=1}^m \xi_{jc} z_j |) = \E_{\xi} (\sqrt{|\sum_{j=1}^m \xi_{jc} z_j |^2}) \leq \sqrt{\E_{\xi} (|\sum_{j=1}^m \xi_{jc} z_j |^2)} = \\& \sqrt{\E_{\xi} (|\sum_{j=1}^m \xi_{jc} z_j} |^2) = \sqrt{\E_{\xi} (|\sum_{j,j'=1}^m \xi_{jc} \xi_{j'c} z_j z_{j'}} |) = \sqrt{\sum_{j=1}^m |z_j|^2} 
\end{aligned}
\end{equation*}
\end{small}
Namely $\E_{\xi}(Z) \leq M_i \sqrt{\sum_{j=1}^m |z_j|^2} $.
$Z$ is a deterministic function of the i.i.d. random variables  $\xi_{jc}$ and satisfies the following:
\begin{small}
\begin{equation*}
\begin{aligned}
Z(\xi_{1c},..., \xi_{jc},...,\xi_{mc}) - Z(\xi_{1c},..., -\xi_{jc},...,\xi_{mc}) \leq 2|z_j|
\end{aligned}
\end{equation*}
\end{small}
This follows from the triangle inequality.
This means that $Z$ satisfies a bounded-difference condition, which, by the proof of Theorem~6.2 in ~\citep{Boucheron2010}, implies that $Z$ is sub-Gaussian, with variance factor:
\begin{small}
\begin{equation*}
\begin{aligned}
v = \frac{1}{4} \sum_{j=1}^m (2M_i|z_j|)^2 = M_i^2 \sum_{j=1}^m |z_j|^2
\end{aligned}
\end{equation*}
\end{small}
It follows that:
\begin{small}
\begin{equation*}
\begin{aligned}
&\frac{1}{\lambda_i} \log \E_{\xi}\left[\exp( \lambda_i Z - \lambda_i \E_{\xi}(Z)) \right] \leq \\& \frac{1}{\lambda_i} \frac{\lambda_i^2 M_i^2 \sum_{j=1}^m |z_j|^2}{2} =  \frac{\lambda_i M_i^2 \sum_{j=1}^m |z_j|^2}{2}  
\end{aligned}
\end{equation*}
\end{small}
Therefore:
\begin{small}
\begin{equation}\label{eq:eq2}
\begin{aligned}
&\frac{1}{\lambda_i} \log \E_{\xi}\left[ \exp( \lambda_i Z - \lambda_i \E_{\xi}(Z)) \right] +  \E_{\xi} (Z) \\
&\leq \frac{\lambda_i M_i^2 \sum_{j=1}^m |z_j|^2}{2} + M_i\sqrt{\sum_{j=1}^m |z_j|^2} 
\end{aligned}
\end{equation}
\end{small}
\paragraph{Analyzing Lipschitz constants $l_{ij}$.}
Next, we analyze the Lipschitz constants \( l_{ij} \).
For $l_{ij} = \max_{z \in dom(\sigma_{ij})}(\exp(z)X_{(j)_{sL}}X_{(j)_{s'i}} X_{(j)_{di}})$ and $l_i = \max(l_{ij})$ is of the form $g_i(X_{(j)})=\sum_{k=i+1}^L \sigma(S^{}_{\Delta} X^{}_{{(j)}_{*k}}) \cdot a^{}_{dl}$ (see \eqref{eq:lip}). Since  
\begin{small}
\begin{equation*}
\begin{aligned}
&l_i = \max_{j} l_{ij} = \max_{j} \max_{z \in dom(\sigma_{ij})}(\exp(z)X^{(j)}_{sL}X^{(j)}_{s'i} X^{(j)}_{di}) \\& \leq \max_{j} \max_{z \in dom(\sigma_{ij})} \exp \left( \sum_{k=i+1}^L \sigma(S^{}_{\Delta} X^{}_{{(j)}_{*k}}) \cdot a^{}_{dl} \right) \cdot 1 < K^{L-i}
\end{aligned}
\end{equation*}
\end{small}
which is followed from our assumptions. 


We get:
\begin{small}
\begin{equation*}
\begin{aligned}
\sum_{i=1}^L l_i (L-i) &\leq \sum_{i=1}^{L} (L-i) (K)^{L-i} = \frac{(L-1)K^{L+1} - L K^L +K}{(K-1)^2}  \\
\end{aligned}
\end{equation*}
\end{small}

We conclude that:
\begin{small}
\begin{equation*}
\begin{aligned}
\lim_{L \rightarrow \infty} \frac{(L-1)K^{L+1} - L K^L +K}{(K-1)^2}= \frac{K}{(K-1)^2}
\end{aligned}
\end{equation*}
\end{small}

\paragraph{Concluding the proof.} By combining ~\eqref{eq:eq1},~\eqref{eq:lip} and~\eqref{eq:eq2}, we have:
\begin{small}
\begin{equation*}
\begin{aligned}
& { \Theta} \leq  \frac{1}{\lambda_i} \log \left(4 L \mathcal{C} D^4 N \right) + \max_{c, d, l, s, s', t, k} \frac{1}{\lambda_i} \cdot \\& \log \left(  \E_{\xi}\left[ \exp \left(\lambda_i \sqrt{D} \rho^{}_A \rho^{}_{\Delta} (L-i) l_i  \left|  \sum_{j=1}^m \xi_{jc} 
X^{}_{{(j)}_{tk}}  \right| \right)\right] \right) \\
&\leq \frac{1}{\lambda_i}  \log \left(4L \mathcal{C} D^4 N \right) +  \max_{c, d, l, s, s', t, k} \frac{\lambda_i (\sqrt{D} \rho^{}_A \rho^{}_{\Delta} (L-i) l_i)^2 \sum_{j=1}^m (X_{(j)_{tk}})^2}{2} + (\sqrt{D} \rho^{}_A \rho^{}_{\Delta} (L-i) l_i)\sqrt{\sum_{j=1}^m (X_{(j)_{tk}})^2} \\
\end{aligned}
\end{equation*}
\end{small}
We choose $\lambda_i = \sqrt{\frac{2\log (4 L \mathcal{C} D^4 N)}{M_i^2  \max_{t,k}\sum_{j=1}^m (X_{(j)_{tk}})^2}}$ which minimizes the above term and obtain the following inequality:
\begin{small}
\begin{equation*}
\begin{aligned}
&\sum_{i=1}^L \frac{1}{\lambda_i} \log \left(4L \mathcal{C} D^4 N \right) + \max_{t, k} \frac{\lambda_i (\sqrt{D} \rho^{}_A \rho^{}_{\Delta} (L-i) l_i)^2 \sum_{j=1}^m (X_{(j)_{tk}})^2}{2} + (\sqrt{D} \rho^{}_A \rho^{}_{\Delta} (L-i) l_i)\sqrt{\sum_{j=1}^m (X_{(j)_{tk}})^2}\\
&\leq \sum_{i=1}^L (1 + \sqrt{2\log (4L \mathcal{C} D^4 N)})M_i\sqrt{ \max_{t, k}\sum_{j=1}^m (X_{(j)_{tk}})^2}\\
&= \sum_{i=1}^L (1 + \sqrt{2\log (4L \mathcal{C} D^4 N)}) \sqrt{D} \rho^{}_A \rho^{}_{\Delta} (L-i) l_i \cdot \sqrt{ \max_{t, k}\sum_{j=1}^m (X_{(j)_{tk}})^2}\\
&\leq (1 + \sqrt{2\log (4L \mathcal{C} D^4 N)}) \sqrt{D} \rho^{}_A \rho^{}_{\Delta}\sqrt{\max_{t, k} \sum_{j=1}^m (X_{(j)_{tk}})^2} \frac{K}{(K-1)^2}\\
\end{aligned}
\end{equation*}
\end{small}
We conclude that:
\begin{small}
\begin{equation*}
\begin{aligned}
&m\mathcal{R(\mathcal{F}_{\rho})} \\ &\leq 
 D^{2} \Gamma (1 + \sqrt{2\log (4L \mathcal{C} D^4 N)}) \sqrt{\max_{t, k} \sum_{j=1}^m (X_{(j)_{tk}})^2} \frac{K}{(K-1)^2}\\
\end{aligned}
\end{equation*}
\end{small}
\end{proof}
\setcounter{theorem}{4}
\begin{theorem}\label{theorem:genproof}
Let $P$ be a distribution over $\mathbb{R}^{D \times L} \times [C]$ and $\delta > 0$. Let $S = \{( X^{}_{(j)},y_{(j)})\}^{m}_{j=1}$ be a dataset of i.i.d. samples selected from $P$. Assume that $\forall j \in [m]: ||X_{(j)}||_{\max} \leq 1$. Additionally, suppose $\forall k \in [L], d \in [D]:||\bar{A}^{}_{dk}||_{\max} < K < 1$. Then, with probability at least $1-\delta$ over the selection of $S$, for any $f_w \in \mathcal{F}$, 
\begin{small}
\begin{equation*}
\begin{aligned}
&\err_P(f_w) - \fr{1}{m}\sum^{m}_{j=1}\bI[\max_{c \neq c'}(f^c_w( X_{(j)})) + \gamma \geq f^{c'}_w( X^{}_{(j)})] \\& = \err_P(f_w) - \err^\gamma_S(f_w) \leq \frac{2\sqrt{2}}{\gamma m} ({\Gamma(w) } +{\frac{1}{D^2N^2}}) D^{2} \cdot \\& (1 + \sqrt{2\log (4L \mathcal{C} D^4 N)}) \sqrt{\max_{t, k} \sum_{j=1}^m (X_{(j)_{tk}})^2} \frac{K}{(K-1)^2} \\&+ 3\sqrt{\frac{\log(2/\delta)+2\log({D^2N^2\Gamma(w) }+2)}{2m}},
\end{aligned}
\end{equation*}
\end{small}
where the maximum is taken over \(t \in [D]\),  \(k \in [L]\). 

\end{theorem}
\begin{proof}
For aesthetic purposes, we define \( B \) as \( S_B \) and \( C \) as \( S_C \). We want to prove the bound for all $f_w \in \mathcal{F}$ where:
{
\begin{small}
\begin{equation*}
\begin{aligned}
&\mathcal{F} := \\&\{f_w : w = (A, B, C, S_\Delta, W), \forall k \in [L], d \in [D]:||\bar{A}^{}_{dk}||_{\max} < K < 1 \}
\end{aligned}
\end{equation*}
\end{small}}
Let $t \in \mathbb{N}$. Denote:
\begin{small}
\begin{equation*}
\begin{aligned}
&\mathcal{S}(t) := \{f_w \in \mathcal{F} , \Gamma(w) < {\frac{t}{D^2N^2}} \}
\end{aligned}
\end{equation*}
\end{small}
Correspondingly subdivide $\delta$ as:
\begin{small}
\begin{equation*}
\begin{aligned}
&\delta(t) := \frac{\delta}{t(t+1)}
\end{aligned}
\end{equation*}
\end{small}

By Lemma ~\ref{lem:loss_ramp} and Theorem ~\ref{theorem:rad}, with probability at least $1-\delta(t)$:
for any function $f_w \in \mathcal{S}(t) $, we have the following inequality:
\begin{small}
\begin{equation*}
\begin{aligned}
\err_P(f_w) - \err^\gamma_S(f_w) \leq \frac{2\sqrt{2}}{\gamma} \cdot \mathcal{R}(\mathcal{S}(t)) + 3\sqrt{\frac{\log(2/\delta(t))}{2m}}.
\end{aligned}
\end{equation*}
\end{small}
Using the union bound over all possible set $\mathcal{S}(t)$, we establish that the above probabilistic bound holds uniformly for all functions $f_w \in \mathcal{S}(t)$ with probability at least $1 - \delta$.
Hence, let  $f_w \in \mathcal{F}$ with weight vector {$w = (A, B, C, S_{\Delta}, W)$}. We choose the smallest $(t)$ such that, $f_w \in \mathcal{S}(t)$. We have:
\begin{small}
\begin{equation*}
\begin{aligned}
&\err_P(f_w) - \err^\gamma_S(f_w) \leq \frac{2\sqrt{2}}{\gamma} \cdot \mathcal{R}(S(t)) + 3\sqrt{\frac{\log(2/\delta(t))}{2m}}
\\ &= \frac{2\sqrt{2}}{\gamma m} \frac{t}{D^2N^2} D^{2} (1 + \sqrt{2\log (4L \mathcal{C} D^4 N)})\cdot \sqrt{\max_{t, k} \sum_{j=1}^m (X_{(j)_{tk}})^2} \frac{K}{(K-1)^2} + 3\sqrt{\frac{\log(2/\delta)+2\log(t+1)}{2m}}
\\ &\leq \frac{2\sqrt{2}}{\gamma m} ({\Gamma(w) } +{\frac{1}{D^2N^2}}) D^{2} (1 + \sqrt{2\log (4L \mathcal{C} D^4 N)}) \cdot  \sqrt{\max_{t, k} \sum_{j=1}^m (X_{(j)_{tk}})^2} \frac{K}{(K-1)^2}  3\sqrt{\frac{\log(2/\delta)+2\log({D^2N^2\Gamma(w) }+2)}{2m}}
\end{aligned}
\end{equation*}
\end{small}
\end{proof}

%% file: example_paper.bbl
\begin{thebibliography}{49}
\providecommand{\natexlab}[1]{#1}
\providecommand{\url}[1]{\texttt{#1}}
\expandafter\ifx\csname urlstyle\endcsname\relax
  \providecommand{\doi}[1]{doi: #1}\else
  \providecommand{\doi}{doi: \begingroup \urlstyle{rm}\Url}\fi

\bibitem[Ahamed \& Cheng(2024)Ahamed and Cheng]{ahamed2024mambatab}
Ahamed, M.~A. and Cheng, Q.
\newblock Mambatab: A simple yet effective approach for handling tabular data.
\newblock \emph{arXiv preprint arXiv:2401.08867}, 2024.

\bibitem[Ali et~al.(2024)Ali, Zimerman, and Wolf]{ali2024hidden}
Ali, A., Zimerman, I., and Wolf, L.
\newblock The hidden attention of mamba models.
\newblock \emph{arXiv preprint arXiv:2403.01590}, 2024.

\bibitem[Allen-Zhu \& Li(2019)Allen-Zhu and Li]{allen2019can}
Allen-Zhu, Z. and Li, Y.
\newblock Can sgd learn recurrent neural networks with provable generalization?
\newblock \emph{Advances in Neural Information Processing Systems}, 32, 2019.

\bibitem[Baron et~al.(2024)Baron, Zimerman, and Wolf]{baron2024a}
Baron, E., Zimerman, I., and Wolf, L.
\newblock A 2-dimensional state space layer for spatial inductive bias.
\newblock In \emph{The Twelfth International Conference on Learning Representations}, 2024.
\newblock URL \url{https://openreview.net/forum?id=BGkqypmGvm}.

\bibitem[Cirone et~al.(2024)Cirone, Orvieto, Walker, Salvi, and Lyons]{cirone2024theoretical}
Cirone, N.~M., Orvieto, A., Walker, B., Salvi, C., and Lyons, T.
\newblock Theoretical foundations of deep selective state-space models.
\newblock \emph{arXiv preprint arXiv:2402.19047}, 2024.

\bibitem[Cohen-Karlik et~al.(2022{\natexlab{a}})Cohen-Karlik, David, Cohen, and Globerson]{cohen2022implicit}
Cohen-Karlik, E., David, A.~B., Cohen, N., and Globerson, A.
\newblock On the implicit bias of gradient descent for temporal extrapolation.
\newblock In \emph{International Conference on Artificial Intelligence and Statistics}, pp.\  10966--10981. PMLR, 2022{\natexlab{a}}.

\bibitem[Cohen-Karlik et~al.(2022{\natexlab{b}})Cohen-Karlik, Menuhin-Gruman, Giryes, Cohen, and Globerson]{cohen2022learning}
Cohen-Karlik, E., Menuhin-Gruman, I., Giryes, R., Cohen, N., and Globerson, A.
\newblock Learning low dimensional state spaces with overparameterized recurrent neural nets.
\newblock \emph{arXiv preprint arXiv:2210.14064}, 2022{\natexlab{b}}.

\bibitem[De et~al.(2024)De, Smith, Fernando, Botev, Cristian-Muraru, Gu, Haroun, Berrada, Chen, Srinivasan, et~al.]{de2024griffin}
De, S., Smith, S.~L., Fernando, A., Botev, A., Cristian-Muraru, G., Gu, A., Haroun, R., Berrada, L., Chen, Y., Srinivasan, S., et~al.
\newblock Griffin: Mixing gated linear recurrences with local attention for efficient language models.
\newblock \emph{arXiv preprint arXiv:2402.19427}, 2024.

\bibitem[Emami et~al.(2021)Emami, Sahraee-Ardakan, Pandit, Rangan, and Fletcher]{emami2021implicit}
Emami, M., Sahraee-Ardakan, M., Pandit, P., Rangan, S., and Fletcher, A.~K.
\newblock Implicit bias of linear rnns.
\newblock In \emph{International Conference on Machine Learning}, pp.\  2982--2992. PMLR, 2021.

\bibitem[Galanti et~al.(2024)Galanti, Xu, Galanti, and Poggio]{galanti2024norm}
Galanti, T., Xu, M., Galanti, L., and Poggio, T.
\newblock Norm-based generalization bounds for sparse neural networks.
\newblock \emph{Advances in Neural Information Processing Systems}, 36, 2024.

\bibitem[Golowich et~al.(2018)Golowich, Rakhlin, and Shamir]{golowich2018size}
Golowich, N., Rakhlin, A., and Shamir, O.
\newblock Size-independent sample complexity of neural networks.
\newblock In \emph{Conference On Learning Theory}, pp.\  297--299. PMLR, 2018.

\bibitem[Gu \& Dao(2023)Gu and Dao]{gu2023mamba}
Gu, A. and Dao, T.
\newblock Mamba: Linear-time sequence modeling with selective state spaces.
\newblock \emph{arXiv preprint arXiv:2312.00752}, 2023.

\bibitem[Gu et~al.(2021{\natexlab{a}})Gu, Goel, and R{\'e}]{gu2021efficiently}
Gu, A., Goel, K., and R{\'e}, C.
\newblock Efficiently modeling long sequences with structured state spaces.
\newblock \emph{arXiv preprint arXiv:2111.00396}, 2021{\natexlab{a}}.

\bibitem[Gu et~al.(2021{\natexlab{b}})Gu, Johnson, Goel, Saab, Dao, Rudra, and R{\'e}]{gu2021combining}
Gu, A., Johnson, I., Goel, K., Saab, K., Dao, T., Rudra, A., and R{\'e}, C.
\newblock Combining recurrent, convolutional, and continuous-time models with linear state space layers.
\newblock \emph{Advances in neural information processing systems}, 34:\penalty0 572--585, 2021{\natexlab{b}}.

\bibitem[Gupta et~al.(2022{\natexlab{a}})Gupta, Gu, and Berant]{dss}
Gupta, A., Gu, A., and Berant, J.
\newblock Diagonal state spaces are as effective as structured state spaces.
\newblock \emph{Advances in Neural Information Processing Systems}, 35:\penalty0 22982--22994, 2022{\natexlab{a}}.

\bibitem[Gupta et~al.(2022{\natexlab{b}})Gupta, Mehta, and Berant]{gupta2022simplifying}
Gupta, A., Mehta, H., and Berant, J.
\newblock Simplifying and understanding state space models with diagonal linear rnns.
\newblock \emph{arXiv preprint arXiv:2212.00768}, 2022{\natexlab{b}}.

\bibitem[Hardt et~al.(2018)Hardt, Ma, and Recht]{hardt2018gradient}
Hardt, M., Ma, T., and Recht, B.
\newblock Gradient descent learns linear dynamical systems.
\newblock \emph{Journal of Machine Learning Research}, 19\penalty0 (29):\penalty0 1--44, 2018.

\bibitem[Hua et~al.(2022)Hua, Dai, Liu, and Le]{hua2022transformer}
Hua, W., Dai, Z., Liu, H., and Le, Q.
\newblock Transformer quality in linear time.
\newblock In \emph{International Conference on Machine Learning}, pp.\  9099--9117. PMLR, 2022.

\bibitem[Jelassi et~al.(2024)Jelassi, Brandfonbrener, Kakade, and Malach]{jelassi2024repeat}
Jelassi, S., Brandfonbrener, D., Kakade, S.~M., and Malach, E.
\newblock Repeat after me: Transformers are better than state space models at copying.
\newblock \emph{arXiv preprint arXiv:2402.01032}, 2024.

\bibitem[Langley(2000)]{langley00}
Langley, P.
\newblock Crafting papers on machine learning.
\newblock In Langley, P. (ed.), \emph{Proceedings of the 17th International Conference on Machine Learning (ICML 2000)}, pp.\  1207--1216, Stanford, CA, 2000. Morgan Kaufmann.

\bibitem[Ledoux \& Talagrand(1991)Ledoux and Talagrand]{Ledoux1991ProbabilityIB}
Ledoux, M. and Talagrand, M.
\newblock \emph{Probability in Banach spaces}.
\newblock Springer Berlin, Heidelberg, 1991.

\bibitem[Li et~al.(2025{\natexlab{a}})Li, Gong, Yang, Shan, Liu, Zhu, Zhang, Guo, Chen, Li, et~al.]{li2025minimax}
Li, A., Gong, B., Yang, B., Shan, B., Liu, C., Zhu, C., Zhang, C., Guo, C., Chen, D., Li, D., et~al.
\newblock Minimax-01: Scaling foundation models with lightning attention.
\newblock \emph{arXiv preprint arXiv:2501.08313}, 2025{\natexlab{a}}.

\bibitem[Li et~al.(2025{\natexlab{b}})Li, Li, Wang, He, Wang, Wang, and Qiao]{li2025videomamba}
Li, K., Li, X., Wang, Y., He, Y., Wang, Y., Wang, L., and Qiao, Y.
\newblock Videomamba: State space model for efficient video understanding.
\newblock In \emph{European Conference on Computer Vision}, pp.\  237--255. Springer, 2025{\natexlab{b}}.

\bibitem[Li et~al.(2024)Li, Singh, and Grover]{mambaND}
Li, S., Singh, H., and Grover, A.
\newblock Mamba-nd: Selective state space modeling for multi-dimensional data.
\newblock \emph{arXiv preprint arXiv:2402.05892}, 2024.

\bibitem[Liang et~al.(2024)Liang, Zhou, Wang, Zhu, Xu, Zou, Ye, and Bai]{mambaPoint}
Liang, D., Zhou, X., Wang, X., Zhu, X., Xu, W., Zou, Z., Ye, X., and Bai, X.
\newblock Pointmamba: A simple state space model for point cloud analysis.
\newblock \emph{arXiv preprint arXiv:2402.10739}, 2024.

\bibitem[Lieber et~al.(2024)Lieber, Lenz, Bata, Cohen, Osin, Dalmedigos, Safahi, Meirom, Belinkov, Shalev-Shwartz, et~al.]{lieber2024jamba}
Lieber, O., Lenz, B., Bata, H., Cohen, G., Osin, J., Dalmedigos, I., Safahi, E., Meirom, S., Belinkov, Y., Shalev-Shwartz, S., et~al.
\newblock Jamba: A hybrid transformer-mamba language model.
\newblock \emph{arXiv preprint arXiv:2403.19887}, 2024.

\bibitem[Liu \& Li(2024)Liu and Li]{liu2024generalization}
Liu, F. and Li, Q.
\newblock From generalization analysis to optimization designs for state space models.
\newblock \emph{arXiv preprint arXiv:2405.02670}, 2024.

\bibitem[Liu et~al.(2024{\natexlab{a}})Liu, Yang, Zhou, Xi, Yu, Yu, Liang, Shi, Zhang, Zheng, et~al.]{mambaMedical3}
Liu, J., Yang, H., Zhou, H.-Y., Xi, Y., Yu, L., Yu, Y., Liang, Y., Shi, G., Zhang, S., Zheng, H., et~al.
\newblock Swin-umamba: Mamba-based unet with imagenet-based pretraining.
\newblock \emph{arXiv preprint arXiv:2402.03302}, 2024{\natexlab{a}}.

\bibitem[Liu et~al.(2024{\natexlab{b}})Liu, Tian, Zhao, Yu, Xie, Wang, Ye, and Liu]{mambaViT1}
Liu, Y., Tian, Y., Zhao, Y., Yu, H., Xie, L., Wang, Y., Ye, Q., and Liu, Y.
\newblock Vmamba: Visual state space model.
\newblock \emph{arXiv preprint arXiv:2401.10166}, 2024{\natexlab{b}}.

\bibitem[Lu et~al.(2021)Lu, Yao, Zhang, Zhu, Xu, Gao, Xu, Xiang, and Zhang]{lu2021soft}
Lu, J., Yao, J., Zhang, J., Zhu, X., Xu, H., Gao, W., Xu, C., Xiang, T., and Zhang, L.
\newblock Soft: Softmax-free transformer with linear complexity.
\newblock \emph{Advances in Neural Information Processing Systems}, 34:\penalty0 21297--21309, 2021.

\bibitem[Lv et~al.(2024)Lv, Deng, Chen, Wang, and Nie]{lv2024decision}
Lv, Q., Deng, X., Chen, G., Wang, M.~Y., and Nie, L.
\newblock Decision mamba: A multi-grained state space model with self-evolution regularization for offline rl.
\newblock \emph{arXiv preprint arXiv:2406.05427}, 2024.

\bibitem[Ma et~al.(2022)Ma, Zhou, Kong, He, Gui, Neubig, May, and Zettlemoyer]{ma2022mega}
Ma, X., Zhou, C., Kong, X., He, J., Gui, L., Neubig, G., May, J., and Zettlemoyer, L.
\newblock Mega: moving average equipped gated attention.
\newblock \emph{arXiv preprint arXiv:2209.10655}, 2022.

\bibitem[Merrill et~al.(2024)Merrill, Petty, and Sabharwal]{merrill2024illusion}
Merrill, W., Petty, J., and Sabharwal, A.
\newblock The illusion of state in state-space models.
\newblock \emph{arXiv preprint arXiv:2404.08819}, 2024.

\bibitem[Qin et~al.(2023)Qin, Li, Sun, Sun, Shen, Han, Wei, Lv, Luo, Qiao, et~al.]{qin2023transnormerllm}
Qin, Z., Li, D., Sun, W., Sun, W., Shen, X., Han, X., Wei, Y., Lv, B., Luo, X., Qiao, Y., et~al.
\newblock Transnormerllm: A faster and better large language model with improved transnormer.
\newblock 2023.

\bibitem[Ramapuram et~al.(2024)Ramapuram, Danieli, Dhekane, Weers, Busbridge, Ablin, Likhomanenko, Digani, Gu, Shidani, et~al.]{ramapuram2024theory}
Ramapuram, J., Danieli, F., Dhekane, E., Weers, F., Busbridge, D., Ablin, P., Likhomanenko, T., Digani, J., Gu, Z., Shidani, A., et~al.
\newblock Theory, analysis, and best practices for sigmoid self-attention.
\newblock \emph{arXiv preprint arXiv:2409.04431}, 2024.

\bibitem[Sarrof et~al.(2024)Sarrof, Veitsman, and Hahn]{sarrof2024expressive}
Sarrof, Y., Veitsman, Y., and Hahn, M.
\newblock The expressive capacity of state space models: A formal language perspective.
\newblock \emph{arXiv preprint arXiv:2405.17394}, 2024.

\bibitem[Shen et~al.(2024)Shen, Li, Leng, Qin, Sun, and Zhong]{shen2024scaling}
Shen, X., Li, D., Leng, R., Qin, Z., Sun, W., and Zhong, Y.
\newblock Scaling laws for linear complexity language models.
\newblock In \emph{Proceedings of the 2024 Conference on Empirical Methods in Natural Language Processing}, pp.\  16377--16426, 2024.

\bibitem[Smith et~al.(2022)Smith, Warrington, and Linderman]{smith2022simplified}
Smith, J.~T., Warrington, A., and Linderman, S.~W.
\newblock Simplified state space layers for sequence modeling.
\newblock \emph{arXiv preprint arXiv:2208.04933}, 2022.

\bibitem[Stéphane~Boucheron(2010)]{Boucheron2010}
Stéphane~Boucheron, Olivier~Bousquet, G.~L.
\newblock Theory of classification: a survey of some recent advances.
\newblock \emph{ESAIM: Probability and Statistics}, 9:\penalty0 323--375, 3 2010.
\newblock URL \url{http://eudml.org/doc/104340}.

\bibitem[Sun et~al.(2023)Sun, Dong, Huang, Ma, Xia, Xue, Wang, and Wei]{sun2023retentive}
Sun, Y., Dong, L., Huang, S., Ma, S., Xia, Y., Xue, J., Wang, J., and Wei, F.
\newblock Retentive network: A successor to transformer for large language models.
\newblock \emph{arXiv preprint arXiv:2307.08621}, 2023.

\bibitem[Touvron et~al.(2021)Touvron, Cord, Douze, Massa, Sablayrolles, and J{\'e}gou]{touvron2021training}
Touvron, H., Cord, M., Douze, M., Massa, F., Sablayrolles, A., and J{\'e}gou, H.
\newblock Training data-efficient image transformers \& distillation through attention.
\newblock In \emph{International conference on machine learning}, pp.\  10347--10357. PMLR, 2021.

\bibitem[Vaswani et~al.(2017)Vaswani, Shazeer, Parmar, Uszkoreit, Jones, Gomez, Kaiser, and Polosukhin]{NIPS2017_3f5ee243}
Vaswani, A., Shazeer, N., Parmar, N., Uszkoreit, J., Jones, L., Gomez, A.~N., Kaiser, L.~u., and Polosukhin, I.
\newblock Attention is all you need.
\newblock In \emph{Advances in Neural Information Processing Systems}, volume~30, 2017.

\bibitem[Waleffe et~al.(2024)Waleffe, Byeon, Riach, Norick, Korthikanti, Dao, Gu, Hatamizadeh, Singh, Narayanan, et~al.]{waleffe2024empirical}
Waleffe, R., Byeon, W., Riach, D., Norick, B., Korthikanti, V., Dao, T., Gu, A., Hatamizadeh, A., Singh, S., Narayanan, D., et~al.
\newblock An empirical study of mamba-based language models.
\newblock \emph{arXiv preprint arXiv:2406.07887}, 2024.

\bibitem[Wang et~al.(2024)Wang, Tsepa, Ma, and Wang]{mambaGraph1}
Wang, C., Tsepa, O., Ma, J., and Wang, B.
\newblock Graph-mamba: Towards long-range graph sequence modeling with selective state spaces.
\newblock \emph{arXiv preprint arXiv:2402.00789}, 2024.

\bibitem[Wortsman et~al.(2023)Wortsman, Lee, Gilmer, and Kornblith]{wortsman2023replacing}
Wortsman, M., Lee, J., Gilmer, J., and Kornblith, S.
\newblock Replacing softmax with relu in vision transformers.
\newblock \emph{arXiv preprint arXiv:2309.08586}, 2023.

\bibitem[Zhang et~al.(2024)Zhang, Liu, Yang, Wang, Chen, Hou, Ren, and Yang]{zhang2024secure}
Zhang, J., Liu, J., Yang, X., Wang, Y., Chen, K., Hou, X., Ren, K., and Yang, X.
\newblock Secure transformer inference made non-interactive.
\newblock \emph{Cryptology ePrint Archive}, 2024.

\bibitem[Zimerman \& Wolf(2023)Zimerman and Wolf]{zimerman2023long}
Zimerman, I. and Wolf, L.
\newblock On the long range abilities of transformers.
\newblock \emph{arXiv preprint arXiv:2311.16620}, 2023.

\bibitem[Zimerman et~al.(2023)Zimerman, Baruch, Drucker, Ezov, Soceanu, and Wolf]{zimerman2023converting}
Zimerman, I., Baruch, M., Drucker, N., Ezov, G., Soceanu, O., and Wolf, L.
\newblock Converting transformers to polynomial form for secure inference over homomorphic encryption.
\newblock \emph{arXiv preprint arXiv:2311.08610}, 2023.

\bibitem[Zuo et~al.(2024)Zuo, Velikanov, Rhaiem, Chahed, Belkada, Kunsch, and Hacid]{zuo2024falcon}
Zuo, J., Velikanov, M., Rhaiem, D.~E., Chahed, I., Belkada, Y., Kunsch, G., and Hacid, H.
\newblock Falcon mamba: The first competitive attention-free 7b language model.
\newblock \emph{arXiv preprint arXiv:2410.05355}, 2024.

\end{thebibliography}
